\documentclass[preprint,12pt]{elsarticle}

\usepackage{amssymb}
\usepackage{amsmath}
\usepackage{amsthm}
\usepackage{pgffor}
\usepackage{multirow}
\usepackage{graphicx}
\usepackage{subfigure}
\usepackage{caption}
\usepackage{bm}
\usepackage{comment}
\usepackage{enumitem}
\usepackage[misc]{ifsym}
\usepackage[utf8]{inputenc}
\usepackage[T1]{fontenc}   
\usepackage[colorlinks,bookmarksopen,bookmarksnumbered]{hyperref}
\usepackage{url}           
\usepackage{booktabs}      
\usepackage{amsfonts}      
\usepackage{nicefrac}      
\usepackage{microtype}     
\usepackage{xcolor}        
\usepackage{algorithm}
\usepackage{algorithmic}
\usepackage[algo2e,algoruled,boxed,lined]{algorithm2e}
\usepackage{mathtools}
\usepackage{pifont}
\usepackage{natbib}

\newtheorem{thm}{Theorem}
\newtheorem{lem}[thm]{Lemma}
\newtheorem{prop}[thm]{Proposition}

\newtheorem{defn}[thm]{Definition}

\newtheorem{rem}[thm]{Remark}
\newcommand{\cmark}{\ding{51}}%
\newcommand{\xmark}{\ding{55}}%

\DeclareMathOperator{\sgn}{sgn}

\journal{Artificial Intelligence}

\begin{document}

\begin{frontmatter}

\title{One-shot Machine Teaching: Cost Very Few Examples to Converge Faster}

\address[inst1]{School of Artificial Intelligence, Jilin University, Changchun, 130012, China}

\address[inst2]{Centre for Frontier AI Research,  A*STAR, 138632, Singapore}

\author[inst1]{Chen Zhang}
\ead{zchen22@mails.jlu.edu.cn}

\author[inst1]{Xiaofeng Cao\corref{cor1}}
\ead{xiaofengcao@jlu.edu.cn}
\cortext[cor1]{Corresponding author: Dr. Xiaofeng Cao.}

\author[inst1]{Yi Chang}
\ead{yichang@jlu.edu.cn}

\author[inst2]{Ivor W Tsang}
\ead{ivor_tsang@ihpc.a-star.edu.sg}

\begin{abstract}

Artificial intelligence is to teach machines to take actions like humans. To achieve intelligent teaching, the machine learning community becomes to think about a promising topic named machine teaching where the teacher is to design the optimal (usually minimal) teaching set given a target model and a specific learner. However, previous works usually require numerous teaching examples along with large iterations to guide learners to converge, which is costly. In this paper, we consider a more intelligent teaching paradigm named one-shot machine teaching which costs fewer examples to converge faster. Different from typical teaching, this advanced paradigm establishes a tractable mapping from the teaching set to the model parameter. Theoretically, we prove that this mapping is surjective, which serves to an existence guarantee of the optimal teaching set. Then, relying on the surjective mapping from the teaching set to the parameter, we develop a design strategy of the optimal teaching set under appropriate settings, of which two popular efficiency metrics, teaching dimension and iterative teaching dimension are one. Extensive experiments verify the efficiency of our strategy and further demonstrate the intelligence of this new teaching paradigm.
\end{abstract}


\end{frontmatter}

\section{Introduction}
\label{intro}

Artificial intelligence \cite{poole1998computational, nilsson1998artificial, luger2005artificial, russell2010artificial} is to teach machines to be equipped with the ability including understanding, synthesizing, and inferring like humans. To achieve such intelligent teaching \cite{musial2012comment, rivest1995being, cakmak2014eliciting}, the machine learning community follows the human teaching paradigm to ponder an inverse problem to machine learning, which is termed as machine teaching (MT) \cite{zhu2015machine, zhu2018overview}: searching the optimal (usually minimal) teaching set\footnote{Teacher collects a set of teaching examples to supervise the training of learner, and the teaching set is a training set for learner.} given a target model and a specific learner. In real world, MT as an intelligent paradigm has benefits in a wide variety of applications, e.g., crowd sourcing \cite{singla2013actively, singla2014near}, transfer learning \cite{pan2009survey}, education and personnel training of human \cite{patil2014optimal, zhu2015machine} and cyber-security \cite{mei2015using, alfeld2016data, alfeld2017explicit}.

Considering the interaction fashion between teachers and learners, MT is studied from either traditional \cite{goldman1996teaching,zhu2013machine} or iterative \cite{liu2017iterative} paradigms. Regarding learners to be non-gradient-based, traditional teachers invoke batch sampling with prior knowledge on the underlying parameter distribution \cite{zhu2013machine,zhu2015machine,liu2016teaching}, then give a teaching set to learners without further communication (leave training to learners own). It can be analogous to a real-world case that a zoologist only provides several representative animal images for a student once to teach how to distinguish the cheetah from the cat, and there is not following interaction between them. The minimal count of provided images is called teaching dimension \cite{goldman1995complexity}. Differently, capturing the iterative manner of model training \cite{liu2017iterative}, iterative teachers sequentially feed teaching sets to gradient-based learners, which is more applicable particularly in large-scale scenarios \cite{xu2021locality} as it produces a concrete and practicable iterative algorithm (teaching strategy). Recall the animal recognition instance, an iterative zoologist would provide images in rounds based on current mastery degree of learners.  The number of rounds is called iterative teaching dimension \cite{liu2017iterative}.

However, contemporary strategies require large iterations as well as numerous examples. This is far from one ultimate goal of MT from a perspective of teaching cost, one-shot iteration: takes one, or a few examples for one-iteration convergence \cite{polyak1964some, fei2006one}. Besides, this objective is also the requirement of high efficiency and low costs for diverse artificial intelligence applications, e.g., autonomous driving \cite{suchan2021commonsense, reece1995control, xie2022distributed} and smart robots \cite{lemaignan2017artificial, alomari2022online}, and serves as a stepping stone towards applying MT in more fields, e.g., few/one-shot learning \cite{ravi2016optimization, snell2017prototypical} and imitation learning \cite{duan2017one}.

\begin{figure}[t]
	\vskip -0.1in
	\centering\includegraphics[width=\columnwidth]{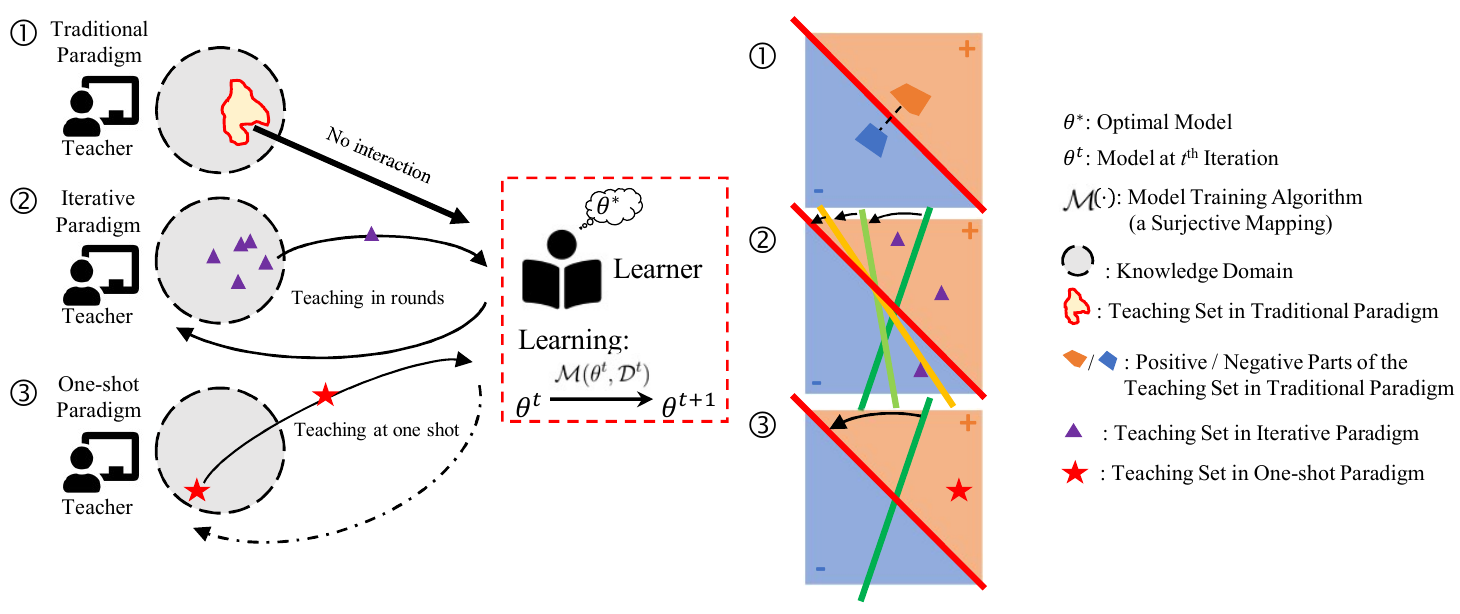}
	\caption{Machine teaching for a binary classification task. Traditional teacher feeds teaching sets to the learner unilaterally without further interaction; iterative teacher provides teaching sets iteratively and requires large iterations; one-shot teacher only provides a singleton to help learner converge within one iteration.}
	\label{OSTS_diff}
\end{figure}

To achieve that goal, we propose a new paradigm named one-shot machine teaching, which requires the optimal teaching set to be a singleton that can guide learners to converge after one iteration. A brief comparison among three paradigms is in Figure \ref{OSTS_diff}. This advanced paradigm establishes a mapping \cite{halmos1960naive, kunen2014set} from the teaching set to the model parameter, which allows us to analyze the essential relation between teaching sets and model parameters. Further, we theoretically prove that this mapping is surjective \cite{bourbaki2006theorie}, which indicates there exists a singleton (i.e., the optimal teaching set) such that the target model can be learnt after one iteration (mapping once).  Then, relying on the surjective mapping from teaching sets to model parameters, we present a strategy called one-shot teaching strategy to solve for the optimal teaching set, of which teaching dimension and iterative teaching dimension are one. By exploring chain expansion of gradient across examples and parameters \cite{bentler1978matrix}, our strategy analytically derives a canonical singleton as the optimal teaching set \cite{crammer2006online} under proper settings and tells the construction approach for teachers with different teaching set candidates (a.k.a. the knowledge domain). We eventually verify efficiency of our strategy with extensive experiments, which demonstrates the intelligence of one-shot machine teaching.

To summarize, the contributions of this work are as follows.
\begin{itemize}
	\vspace{-3pt}
	\item We study a new paradigm named One-shot Machine Teaching (OSMT), which is to search a singleton to achieve one ultimate goal of machine teaching: taking fewer examples to aggressively converge within one iteration. OSMT establishes a mapping proved to be surjective from teaching sets to model parameters, which brings convenience to analyze the fundamental relation between teaching sets and model parameters.
	\vspace{-3pt}
	\item We provide a solution called One-shot Teaching Strategy (OSTS) to design the optimal teaching set for different teachers under proper settings. Specifically, with surjective mapping from teaching sets to model parameters, OSTS exploits chain expansion of gradients and analytically derives a canonical singleton as the optimal teaching set, of which teaching dimension and iterative teaching dimension are one.
	\vspace{-3pt}
	\item We empirically validate the efficiency of OSTS: costs few examples and achieve one-iteration convergence. This demonstrates the fascinating evidence that intelligent one-shot machine teaching can strikingly improve the efficiency of teaching.
	\vspace{-8pt}
\end{itemize}

\section{Related Works}\label{rwk}

This paper is closely related to Machine Teaching (MT) especially Iterative Machine Teaching \cite{liu2017iterative}. We review the development of MT including traditional and iterative paradigms, and we situate this work in it. Besides, we also discuss the connection and difference between this work and some relevant one-shot machine learning works as they mention the same terminology, one-shot.

\subsection{Machine teaching}

As \citet{rivest1995being} pointed that being taught is faster than learning, MT considers an inverse problem to machine learning by following the human teaching paradigm. The promising application value of MT has been verified over diverse domains including human education \cite{khan2011humans, johns2015becoming}, cyber security \cite{barreno2010security, mei2015using, mei2015security} and human-machine interaction \cite{brochu2010bayesian, cakmak2011mixed, cakmak2014eliciting, meek2016analysis, schmit2018human}, reinforcement learning \cite{chuang2020using,zhang2020sample}. 

Traditional paradigms \cite{shinohara1991teachability, goldman1996teaching, ben1998self,zhu2013machine} towards lower teaching cost began an early exploration where teachers are assumed to unilaterally convey teaching sets once (without further communication) and know the underlying parameter distribution and feature space. The most popular quantity of teaching cost is the teaching dimension \cite{goldman1995complexity} which characterizes the size of the teaching set. Meanwhile, \citet{zilles2008teaching,  zilles2011models} and \citet{chen2018near} analyze the machine teaching problem in a theoretical sight as well. Moreover, \citet{balbach2009recent} summarize developments in the algorithmic teaching, and \citet{zhu2015machine} presents a general teaching framework closely connecting to curriculum learning \cite{bengio2009curriculum, khan2011humans, lessard2019optimal} and knowledge distillation \cite{hinton2015distilling}. The most related work in the traditional paradigm of this paper is \citet{liu2016teaching}. It discusses the teaching dimension of linear learners, but is not aware of the iterative interaction between teachers and learners, or produces a practical iterative teaching strategy, or considers pool-based teachers, all of which are considered in this work. Besides, Theorem \ref{ost} in this paper shows that we touch a tighter bound of teaching dimension.

Differently, capturing the iterative fashion of model training \cite{liu2017iterative}, iterative paradigms explore towards one-shot iteration further, in which teachers feed teaching sets to gradient-based learners step by step. The teaching cost is quantified by iterative teaching dimension \cite{liu2017iterative}. This type of paradigm is more practicable in complicated (e.g., large-scale) tasks \cite{xu2021locality} as it introduces an iterative algorithm (teaching strategy) to simplify and solve problems gradually \cite{liu2021iterative}. Besides, \citet{liu2018towards} propose cross-space teaching towards black-box \cite{angluin1997teachers} scenarios, and \citet{cicalese2020teaching} focus on weaker black-box learners and provide their insight into limited information teaching. At the meantime, \citet{hunziker2018teaching} and \citet{dasgupta2019teaching} study the forgetful case. Moreover, \citet{zhou2020crowd} use iterative machine teaching to analyze crowd teaching with imperfect labels, and \citet{mosqueira2021integrating} apply the combination of iterative machine teaching and active learning in machine learning loop, which implies the close relation between iterative machine teaching and active learning \cite{gopfert2020intuitiveness}. There has been a recent growth of interest in iterative paradigms, which covers diverse fields including end-to-end teaching \cite{fan2018learning}, inverse reinforcement learning \cite{kamalaruban2019interactive}, text classification \cite{mallinar2020iterative}, Internet-of-Things \cite{xu2021locality}, human-robot interaction \cite{losey2018responding,yeo2019iterative,han2019robust,sanchez2021people}, computer vision \cite{wang2021gradient,wang2021machine}. Compared with contemporary iterative works, our intelligent one-shot paradigm dramatically cuts down the teaching cost (lower iterative teaching dimension) and achieves one-shot iteration.

\subsection{One-shot machine learning}

One-shot learning attempts to train models from a limited number of examples \cite{vinyals2016matching,fei2006one,wang2020generalizing}, motivated by a human learning ability (learn new concepts with very few instances). Many works attempt to solve this challenge from diverse directions including encoder training \cite{pan2020teach,bajracharya2020mobile} and example generation \cite{rezende2016one,vinyals2016matching,tjandra2018machine}.

The encoder training methods \cite{schwartz2018delta, chen2019multi} attempt to train an encoder such that it can thoroughly learn the inside characters and structures of each example. Some works \cite{pan2020teach,bajracharya2020mobile} may call such encoder training as teaching. Specifically, teaching in \cite{pan2020teach} is to train a Bidirectional Long Short Term Memory encoder, and \citet{bajracharya2020mobile} let human use virtual reality to take some actions for training the encoder inside robots such that robots can reproduce these behaviors. Essential, such encoder training is to search the optimal parameter of encoders based on the training set with limited volume, which is a typical optimal model parameter searching (i.e., machine learning) procedure \cite{mitchell1997machine}. In contrast, MT is to find the optimal teaching set based on a target model, which is an inverse procedure of machine learning \cite{zhu2015machine}. As a teaching paradigm, proposed OSMT is to search the optimal teaching set with smaller cardinality such that learners can converge after one-iteration training.

The generative approaches \cite{salakhutdinov2012learning, lake2015human} attempt to enrich training sets based on very fewer available examples to meet the training requirement of modern models especially deep neural networks. \citet{rezende2016one} develop a machine learning systems to generate new examples such that there is not significant difference between generated and existent examples. \citet{vinyals2016matching} share the similar idea as augmentation, which attempts to map a small labeled support set and an unlabeled example to its labels. Besides, \citet{tjandra2018machine} combine Tacotron with DeepSpeaker to generate speeches from different speakers. However, whether each generated example is helpful for the convergence of a model is not discussed. Differently, OSTS with the prior knowledge of $\theta^*$ (usually come from the pretrained model or expert experience) generates examples such that they can help learners converge after one iteration, which can be viewed as an inverse procedure of one-shot machine learning \cite{zhu2015machine}.  Fundamentally, the difference comes from different settings between machine learning and machine teaching \cite{zhu2018overview}. Besides, the purpose of generation (e.g., linear combination and scaling) in OSTS for pool-based teachers is for the necessity of synthetic examples \cite{polyak1964some, crammer2006online} rather than similarity between synthetic and existing examples.

\section{Preliminaries}\label{fs}

\subsection{Notations}
Let $\mathcal{X}\subseteq\mathbb{R}^n$ be a $n$ dimensional example (a.k.a. feature and input) space and $\mathcal{Y}\subseteq\mathbb{R}\,(\text{Regression})\,\text{or}\,\mathcal{Y}=\left\{-1,1\right\}(\text{Classification})$ be a label (a.k.a. output) space. The standard basis of $\mathbb{R}^n$ is denoted as $\{e_i|1\leq i\leq n\}$ where $e_i$ means the vector with a 1 in the $i^\text{th}$ coordinate and 0's elsewhere. A teaching example (could be abbreviated as example) is a pair $(x,y)\in\mathcal{X}\times\mathcal{Y}$. The optimal teaching example is $(x^*,y^*)$. A size $k$ teaching set is defined as $\mathcal{D}=\{(x_1,y_1),\dots(x_k,y_k)\}=\{(x_i,y_i)\}_{i=1}^k$ with replacement and $\{(x_i,y_i)\}_{i=1}^0\coloneqq\emptyset$. The optimal one is $\mathcal{D}^*$. We omit the set symbol $\{\cdot\}$ for convenience if $\mathcal{D}$ is a singleton. The collection of teaching set candidates is denoted by $\mathbb{D}\ni\mathcal{D}$, which is also called knowledge domains of teachers \cite{liu2017iterative}. Parameter is denoted by $\theta\in\Theta\subseteq\mathbb{R}^n$, and optimal and initial parameters are $\theta^*$ and $\theta^0$, respectively. A model is a mapping $f(x,\theta): \mathcal{X}\times\Theta\to\mathcal{Y}$ assumed to be identified by its parameter. In this paper, we are mainly concerned with models of linear family $f(\theta,x)=\langle\theta,x\rangle$ (i.e., allow inner product $\langle\cdot,\cdot\rangle$ defined in Euclidean vector space across $\mathcal{X}$ and $\Theta$ since they essentially are the subspace of the Euclidean vector space), and connect it with the generalized linear model \cite{mccullagh2019generalized} by the typical convex loss function $\ell(f,y)$ (e.g., square, logistic and hinge loss). The result of this paper can be generalized to the reproducing kernel Hilbert space $\mathcal{H}$ equipped with $\langle\cdot,\cdot\rangle_\mathcal{H}$ by kernel methods \cite{shawe2004kernel, hofmann2008kernel}.

\subsection{Machine teaching}

Machine teaching considers the selection of teaching set for learners such that the model parameter $\theta^*$ can be learnt with fewer teaching examples, and teachers are only allowed to teach via examples rather than parameter clone directly \cite{zhu2015machine}. It has been framed into a bilevel optimization over $\mathbb{D}$ \cite{zhu2018overview}:
\begin{eqnarray}\label{eq1}
	\begin{aligned}
		&\mathcal{D}^*=\underset{\mathcal{D}\in\mathbb{D}}{\arg\min}\quad \mathcal{B}(\hat{\theta},\theta^*)+\tau \mathcal{C}(\mathcal{D})\\
		&\text{s.t.}\quad\hat{\theta}=\textbf{ML}(\mathcal{D})
	\end{aligned},
\end{eqnarray}
where $\mathcal{B}(\cdot,\theta^*)$ is a class of distance  (e.g., L-2 norm $\mathcal{B}(\hat{\theta},\theta^*)=\|\hat{\theta}-\theta^*\|_2$) measuring the learner bias w.r.t. $\theta^*$, $\mathcal{C}(\cdot)$ is the teaching cost (e.g.,  $\mathcal{C}(\mathcal{D})=\left\|\mathcal{D}\right\|_0\geq1$ with cardinality operator $\left\|\cdot\right\|_0$) regularized by a constant $\tau$, and  \textbf{ML} represents the model training (i.e., machine learning) procedure of learners. With prior knowledge about $\theta^*$, the traditional and iterative paradigms are provably more efficient than machine learning. However, from a perspective of teaching cost, one-shot iteration (taking a smaller teaching set $\mathcal{D}\in\mathbb{D}$ for converging to $\theta^*$ within one iteration) is still under exploration even tradition and iterative paradigms have made great progress towards it.

\subsection{Traditional paradigms}

\begin{figure}[ht]
	\subfigbottomskip=-6pt
	\centering
	\subfigure{\includegraphics[width=0.8\columnwidth]{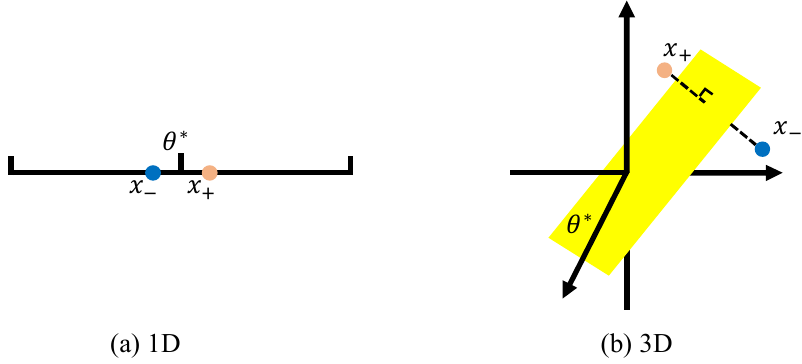}}
	\caption{1D and 3D cases of constructing $\mathcal{D}^*$ under traditional paradigms. $\theta^*$ is the midpoint of line segment $x_-x_+$ in 1D, and vertically bisects $x_-x_+$ in 3D case.}
	\label{traditionalTeacher}
	\vskip -0.1in
\end{figure}

Traditional paradigms where teachers only provide a teaching set without additional interaction with learners explore one-shot iteration via batch sampling with solid prior knowledge including concrete $\theta^*$, example distributions and feature spaces. Considering a 1D threshold classifier teaching, given $\theta^*\in\Theta=\mathbb{R}$ over uniform distribution, and 1D feature space $\mathcal{X}=\mathbb{R}$, teachers pick two teaching examples whose feature vectors satisfy \cite{zhu2018overview}
$$\frac{x_-+x_+}{2} =\theta^*$$
with the negative $x_-$ and the positive $x_+$. Analogously, when teaching a hard-margin support vector machine (SVM) of homogeneous version in high-dimensional space, e.g., 3D cases, let $\theta^*=(a^*,b^*,c^*)^T$, teachers pick $x_+=(a_+,b_+,c_+)^T$ and $x_-=(a_-,b_-,c_-)^T$ satisfying that vector $x_+-x_-$ is orthogonal to hyperplane corresponding to $\theta^*$, and the distance between $x_+$ and $\theta^*$ is the same as that between $x_-$ and $\theta^*$ \cite{zhu2018overview}, which can be expressed as
\begin{eqnarray}
	\begin{aligned}
		\left\{
		\begin{array}{lr}
			a^*(a_+-a_-)+b^*(b_+-b_-)+c^*(c_+-c_-)=0\\	\frac{a^*a_++b^*b_++c^*c_+}{\sqrt{{a^*}^2+{b^*}^2+{c^*}^2}}=-\frac{a^*a_-+b^*b_-+c^*c_-}{\sqrt{{a^*}^2+{b^*}^2+{c^*}^2}}
		\end{array}
		\right.		
	\end{aligned}.\nonumber
\end{eqnarray} 
Above instances are visualized in Figure \ref{traditionalTeacher}.

\subsection{Iterative paradigms}\label{imt}

In contrast, iterative paradigms only requires concrete $\theta^*$ to construct $\mathcal{D}^*$ via iteratively searching. With such weaker requirement and iterative manner, these type of paradigms are more practicable especially when it is computationally infeasible to train over the entire dataset \cite{xu2021locality}. Generally, the iterative paradigm constructs $\mathcal{D}^*$ with focus on a specified algorithm in training procedures of learners \cite{liu2021iterative}. In more detail, it considers $\hat{\theta}=\textbf{ML}(\mathcal{D})$ as
\begin{equation}\label{sgd}
	\hat{\theta}=\arg\underset {\theta\in\Theta}{\min}\mathbb{E}_{(x,y)\sim\mathbb{P}(x,y)}\left[\ell(f(x,\theta),y)\right],
\end{equation}
whose parameter is specifically learnt by mini-batch gradient descent algorithm \cite{li2014efficient,liu2017iterative}:
\begin{equation}\label{eq3}
	\theta^{t+1}\gets\theta^t-\frac{\eta^t}{k}\sum_{i=1}^k\nabla_\theta \ell\left(f(\theta^t,x^t_i), y^t_i\right),
\end{equation}
where $t=0,1,\dots,T$ is the iteration number, $\eta^t>0$ is learning rate (a small constant), $k\geq1$ is the batch size, $x^t_i$ denotes the $i^\text{th}$ example features at the $t^\text{th}$ iteration corresponding to label $y^t_i$ and $\ell(\cdot,\cdot)$ is the convex loss function. For succinctness, Eq.(\ref{eq3}) can be rewritten as 
\begin{equation}\label{scsgd}
	\theta^{t+1}\gets\theta^t-\eta^t\mathcal{G}(\theta^t;\mathcal{D}^t|\ell;f),
\end{equation}
where $\mathcal{G}(\cdot|\ell;f)$ denotes the gradient function when the family of both loss $\ell$ and model $f$ are given.

Specifically, we are concerned with one-shot teaching under applicable \cite{liu2021iterative, xu2021locality} iterative settings where the learner is gradient-based and its training algorithm is iterative. Besides, following the realistic fact that human teachers would teach familiar learners better \cite{feldman1976superior,shishavan2009characteristics}, we focus on the case learners are well-known for teachers, which means the teacher knows above factors (e.g., $\eta$ and $\theta^t$) of gradient-based learners \cite{liu2017iterative, liu2021iterative}. The unfamiliar learner case where teachers do not know some factors can be studied with surrogate methods \cite{queipo2005surrogate,liu2018towards}.

\section{One-shot Teaching}\label{pimt}

In this section, we firstly establish a tractable mapping from teaching sets to model parameters, which is proved to be surjective and corroborates the existence of the optimal teaching set. Since the teaching set is closely related to the knowledge domain of teachers, we, meanwhile, present further discussion on different types of teachers. Then, we give the highly-efficient one-shot teaching strategy (OSTS) for different teachers, which solves for the optimal teaching set analytically and derives a canonical singleton. Finally, we provide the case study about typical learner loss functions, which verifies generalizability of OSTS. 

\subsection{One-shot paradigm}

The one-shot paradigm aims to design a singleton as the optimal teaching set to steer learners to converge after one iteration. It attempts to achieve one-shot iteration by establishing a tractable mapping from teaching sets to model parameters, and only requires concrete $\theta^*$ as well. First of all, we present an expansion insight of the gradient across teaching sets and parameters in model training algorithm Eq.(\ref{scsgd}) based on chain rule.

\begin{prop}\label{epg}
	Given a specific teaching example $(x,y)$, the gradient of loss function $\ell$ w.r.t. the model $f(\theta,x)=\langle\theta,x\rangle$ can be expended as a constant times the example vector based on chain rule:
	\begin{equation}
		\mathcal{G}(\theta;(x,y)|\ell;f)=\nabla_f \ell(f(\theta,x) , y)\cdot x\coloneqq\zeta \cdot x,
	\end{equation}
	where it omits the set symbol $\{\cdot\}$ out of $\{(x,y)\}$ for convenience.
\end{prop}

The teacher is an agent who knows the target model and attempts to search examples for feeding learners. The Learner is an agent who update its parameter based on received teaching set via the model training algorithm Eq.(\ref{scsgd}). Fundamentally, Eq.(\ref{scsgd}) tells the relation between the teaching set (as the input) and the model parameter (as the output), thus we make use of this relation to establish a tractable mapping from teaching sets to model parameters: let $\eta$ and input $\theta$ be fixed, each input $\mathcal{D}$ yields a corresponding output $\theta'_{\mathcal{D}}$, and the assignment is characterized by Eq.(\ref{scsgd}) whose domain is $\mathbb{D}$ and codomain is $\Theta$. Similarly, let $\eta$ and input teaching set $\mathcal{D}$ be fixed, each input $\theta$ yields a corresponding output $\theta'_{\theta}$, and Eq.(\ref{scsgd}) defines the assignment manner, of which domain and codomain are $\Theta$ in this case. These derive Lemma \ref{smp}.

\begin{lem} \label{smp}
	Given a specific learner ($\ell$ and $\eta$ are fixed), the mini-batch gradient descent algorithm (Eq.(\ref{scsgd})) is a tractable mapping $\mathcal{M}(\theta,\mathcal{D}): \Theta\times\mathbb{D}\mapsto\Theta$ satisfying
	\begin{enumerate}[label={\rm (\alph*)}]
		\setlength\itemsep{0.0em}
		\item $\mathcal{M}$ is surjective when $\mathcal{D}$ is fixed: $\forall \theta' \in\Theta,\exists\theta\in\Theta$ such that $\mathcal{M}(\theta|\mathcal{D})=\theta'$
		\item $\mathcal{M}$ is surjective when input $\theta$ is fixed: $\forall \theta' \in\Theta,\exists\mathcal{D}\in\mathbb{D}$ such that $\mathcal{M}(\mathcal{D}|\theta)=\theta'$.
	\end{enumerate}
\end{lem}
\begin{proof} Without loss of generality, we are concerned with $\left\|\mathcal{D}\right\|_0=1$ case for simplicity, i.e., $\mathcal{D}$ has one element $(x,y)$. The proof when $\left\|\mathcal{D}\right\|_0>1$ can be derived analogously.
	
	(a): For proving $\mathcal{M}$ is surjective in its first arguments, we need to find $\theta_{\theta'}$ such that $\theta'=\mathcal{M}(\theta_{\theta'}|\mathcal{D})$. Since $\Theta$ and $\mathcal{X}$ fundamentally are the subset of a vector space $\mathbb{R}^n$, the operations of vector addition and scalar multiplication is well defined. With fixed $(x,y)$, specific $\ell$, $f$ and $\eta$, $\mathcal{M}$ can be specified as 
	\begin{eqnarray}
		\mathcal{M}(\theta|(x,y))=\theta-\eta\zeta x
	\end{eqnarray} based on Proposition \ref{epg}. Therefore, we can derive that $\forall \theta' \in\Theta$, there always exists $\theta_{\theta'}=\theta'+\eta\zeta x$ such that $\mathcal{M}(\theta_{\theta'}|(x,y))=\theta_{\theta'}-\eta\zeta x=\theta'$, which concludes the proof of (a).
	
	(b): Similarly, with fixed $\theta$, $\forall\theta'\in\Theta$, there always exists $x_{\theta'}=(\theta-\theta')/(\eta\zeta)$ (since $\eta$ and $\zeta$ are non-zero scalars) with its label $y_{\theta'}$ such that $\mathcal{M}((x_{\theta'},y_{\theta'})|\theta)=\theta-\eta\zeta x_{\theta'}=\theta'$.
\end{proof}

The Lemma \ref{smp} corroborates that this mapping is surjective in both arguments, which is important as it provides a existence guarantee of the optimal teaching set. Intuitively, the Lemma \ref{smp} tells that
\begin{enumerate}[label= (\alph*)]
	\item For any model $\theta'\in\Theta$ and fixed teaching set $\mathcal{D}$, there always exists a learner whose initial model is $\theta_{\theta'}$ such that this learner can learn $\theta'$ after one-iteration update with $\mathcal{D}$ following Eq.(\ref{scsgd}).
	\item For any model $\theta'\in\Theta$ and a learner with fixed model $\theta$, there always exists a teaching set $\mathcal{D}_{\theta'}$ such that this learner can converge to $\theta'$ after one-iteration evolution with $\mathcal{D}_{\theta'}$ via Eq.(\ref{scsgd}).
\end{enumerate}

The established mapping brings  convenience to study the relation among input $\theta$, $\mathcal{D}$ and output $\theta'$. With the perspective of surjective mapping based on (b) of Lemma \ref{smp} (existence guarantee), the problem has been changed to design the optimal teaching set $\mathcal{D}^*$ such that mapping once can achieve convergence from to pick the optimal teaching set in each iteration to enlarge per-iteration improvement, which is briefly illustrated in Figure \ref{IMTvsOSMT}.There is a natural concern how to construct $\mathcal{D}^*$ when $\theta^0$ and $\theta^*$ are fixed such that the learner can converge to $\theta^*$ after update once, i.e., $\mathcal{M}(\mathcal{D}^*|\theta^0) = \theta^*$. In another word, one-shot machine teaching is to find the singleton root of 
\begin{eqnarray}\label{root}
	\mathcal{F}(\mathcal{D}|\theta^0,\theta^*) \coloneqq \mathcal{M}(\mathcal{D}|\theta^0)-\theta^*,
\end{eqnarray}
and $\mathcal{D}^*=\mathcal{F}^{-1}(0|\theta^0,\theta^*)$ with $\|\mathcal{D}^*\|_0=1$.

\begin{figure}[t]  
	\vskip -0.1in
	\subfigbottomskip=-6pt
	\centering\includegraphics[width=.5\columnwidth]{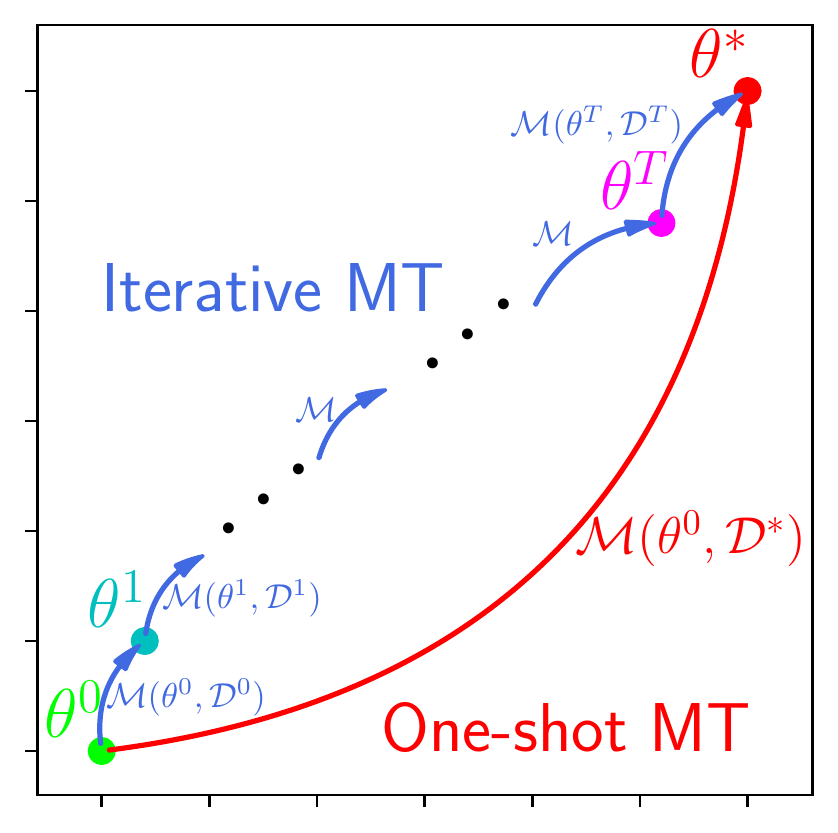}
	\vskip -0.1in
	\caption{Comparison between iterative MT and one-shot MT. Iterative MT feeds $\mathcal{D}^0,\mathcal{D}^1,\dots,\mathcal{D}^T$ gradually and the learner converge to $\theta^*$ iteratively. One-shot MT searches the optimal teaching set $\mathcal{D}^*$, via which the learner converge to $\theta^*$ after one iteration.}
	\label{IMTvsOSMT}
	\vskip -0.2in
\end{figure}

Since provided $\mathcal{D}$ depends on the specific $\mathbb{D}$ of teachers, we further discuss possible teachers with different $\mathbb{D}$. Different from \citet{liu2017iterative,liu2018towards}, we separate teachers into two types based on infinite and finite $\mathbb{D}$, which are the complete teacher and pool-based teachers, respectively.

\begin{defn}[Complete teacher]\label{ct}
	A teacher is called a complete teacher, if the knowledge domain $\mathbb{D}$ of this teacher is derived from the entire teaching example space, an infinite pool $\mathcal{P}=\{(x_i,y_i)|x_i\in\mathcal{X},y_i\in\mathcal{Y}\}$.
\end{defn}
It is trivial to derive that the complete teacher is equipped with addition and scalar multiplication. Let $\|\mathcal{P}\|_0$ in Definition \ref{ct} be finite, we derive the definition of the pool-based teacher:
\begin{defn}[Pool-based teacher]\label{pbt}
	A teacher is called a pool-based teacher, if the knowledge domain $\mathbb{D}$ of this teacher is derived from a pool $\mathcal{P}_N=\{(x_i,y_i)|x_i\in\mathcal{X}_N,y_i\in\mathcal{Y}_N,i\leq N<\infty\}\subsetneq\mathcal{X}\times\mathcal{Y}$.
\end{defn}
Let $\mathcal{X}_\mathcal{P}\times\mathcal{Y}_\mathcal{P}\subset\mathcal{X}\times\mathcal{Y}$ be the completion of $\mathcal{P}_N$ w.r.t. equipped operations. Further, we define three types of pool-based teachers based on such operations.
\begin{defn}[Combinable teacher]
	A pool-based teacher is a combinable teacher, if $\mathcal{X}_\mathcal{P}$ is closed under addition and scalar multiplication, i.e., $\mathcal{X}_\mathcal{P}={\rm span}(\mathcal{X}_N)$: it can provide linear combination of existing examples as new examples.
\end{defn}
\begin{defn}[Scalable teacher]
	A pool-based teacher is a scalable teacher, if $\mathcal{X}_\mathcal{P}$ is not closed under addition but closed under scalar multiplication, i.e., $\mathcal{X}_\mathcal{P}=\alpha\cdot\mathcal{X}_N,\alpha\in\mathbb{R}$: it can provide scaled examples as new examples.
\end{defn}
\begin{defn}[Naive teacher]
	A pool-based teacher is a naive teacher, if $\mathcal{X}_\mathcal{P}$ is not closed under addition or scalar multiplication, i.e., $\mathcal{X}_\mathcal{P}=\mathcal{X}_N$: it can only provide existing examples.
\end{defn}

\begin{table}[b]
	\caption{The relation among different teachers with their properties. "ScaMul" is the abbreviation of "scalar multiplication". \cmark$\,$ below "Complete $\mathbb{D}$" means this type of teacher has complete $\mathbb{D}$, and \xmark$\,$ means not. \cmark$\,$ below "Addition" and "ScaMul" means the example space of this teacher is closed under that operations, and \xmark$\,$ means not.}
	\label{fourt}
	\begin{center}
		\begin{small}
			\begin{tabular}{lcccc}
				\toprule
				Primary                     & Secondary       & Complete $\mathbb{D}$ & Addition& ScaMul \\
				\midrule
				Complete                     & -             & \cmark                & \cmark  & \cmark \\
				\multirow{3}{*}{Pool-based}  &Combinable     & \xmark                & \cmark  & \cmark \\
				&Scalable       & \xmark                & \xmark  & \cmark \\
				&Naive          & \xmark                & \xmark  & \xmark \\
				\bottomrule
			\end{tabular}
		\end{small}
	\end{center}
	\vskip -0.1in
\end{table}

Note that there are infinite potential teaching examples for combinable and scalable teacher as they can generate examples by linear combination or scaling, but the naive teacher has finite examples. The complete teacher can be viewed as an infinite version of the pool-based teacher. We list the property of above mentioned teachers and their relation in Table \ref{fourt}

\subsection{One-shot teaching strategy}\label{ostspf}

For different teachers, with established surjective mapping from teaching sets and parameters we design corresponding one-shot teaching strategies (OSTS), which attempts to analytically solve for a singleton root $\mathcal{D}^*=\{(x^*,y^*)\}$ where $(x^*,y^*)$ is the optimal teaching example. This means to find tighter bound of teaching dimension and iterative teaching dimension.

\begin{thm} [One-shot teachability] \label{ost}
	For a target $\theta^*$, there exists a singleton, on which a learner with learning rate $\eta\neq0$, initial $\theta^0$ and loss function $\ell(f,y)$ can converge to $\theta^*$ after one-iteration training (it can achieve one-shot iteration), if learner $\ell$ is differentiable w.r.t. $f$.
\end{thm}

\begin{proof}To prove the existence of the singleton that can help the learner converge after one iteration, we need to solve for associated $(x^*,y^*)$ in terms of $\theta^*$ in closed form.
	
	With the mapping view in Lemma \ref{smp}, it fixes $\theta^0$ and $\theta^*$ to find a teaching example $(x,y)$ such that the learner touch $\theta^*$ after mapping once:
	\begin{eqnarray}\label{optm}
		\mathcal{M}((x,y)|\theta^0)=\theta^*.
	\end{eqnarray}
	It follows from Eq.(\ref{root}) that it is equivalent to find the singleton root:
	\begin{eqnarray}\label{sroot}
		(x^*,y^*) = \mathcal{F}^{-1}(\mathcal{D}|\theta^0,\theta^*) \qquad \text{s.t.}\,\|\mathcal{D}\|_0 = 1,
	\end{eqnarray}
	where we omit the set symbol $\{\cdot\}$ for convenience. Recall Proposition \ref{epg}, when the first-order partial derivatives of $\ell(f,y)$ w.r.t. $f$ exists, it follows from the chain rule that the gradient at one example can be expanded as
	\begin{eqnarray}\label{eq5}
		\mathcal{G}(\theta;(x,y)|\ell;f)=\nabla_\theta\ell(\left\langle \theta,x\right\rangle , y)=\nabla_{\left<\theta,x\right>} \ell(\left\langle \theta,x\right\rangle , y)x.
	\end{eqnarray}
	Therefore, analytically solving for $(x^*,y^*)$ equals to find the root of following equation
	\begin{equation}\label{eq6}
		\theta^0-\eta\nabla_{\left<\theta,x\right>} \ell\left(\left\langle \theta^0,x\right\rangle , y\right)x-\theta^*=0.
	\end{equation}	
	Since $\theta^0\neq\theta^*$, we deduce that each term of the product $\eta\cdot\nabla_{\left<\theta,x\right>} \ell(\left\langle \theta^0,x\right\rangle , y)\cdot x$ is a non-zero scalar or vector. Thus, from Eq.(\ref{eq6}) we can divide through by $-\eta\nabla_{\left<\theta,x\right>} \ell(\left\langle \theta^0,x\right\rangle , y)$ and rearrange, then obtain
	\begin{eqnarray}\label{eq7}
		x^*=\xi_{\eta;\theta^0}(\theta^*-\theta^0),
	\end{eqnarray}
	where constant $\xi_{\eta;\theta^0}\coloneqq-1/\left(\eta\nabla_{\left<\theta,x\right>} \ell(\left\langle \theta^0,x^*\right\rangle , y^*)\right)$. $y^*$ is a free variable and usually can be derived by $f^*(x^*)$.
	Thereby, we have analytically find the singleton root expressed by $\theta^*$
	\begin{equation}\label{eq9}
		(x^*,y^*)=\left(\xi_{\eta;\theta^0}(\theta^*-\theta^0),f^*(\xi_{\eta;\theta^0}(\theta^*-\theta^0))\right),
	\end{equation}
	namely the universal optimal teaching example. This completes the proof.
\end{proof}

One-shot teachability is a weaker condition compared with Lipschitz smoothness and strongly convexity since differentiable is weaker than these regularities. For instance, absolute loss is one-shot teachable but not Lipschitz smooth and strongly convex. Moreover, as a side effect, it derives that the teaching dimension and iterative teaching dimension of those strategies who provide concrete construction methods for existent singleton in Theorem \ref{ost} are one, which is tighter than that of exponential teachability \cite{liu2017iterative,liu2018towards,liu2021iterative}.

Since $\mathbb{D}$ of the complete teacher is derived from the entire teaching example space, it is straightforward to derive OSTS for this type of teacher:

\begin{prop}[OSTS for the complete teacher]\label{costs}
	For a learner with $\eta\neq0$ and $\theta^0$, the complete teacher should provide the universal optimal teaching example 
	\begin{eqnarray}
		(x^*,y^*)=\left(\xi_{\eta;\theta^0}(\theta^*-\theta^0),f^*(\xi_{\eta;\theta^0}(\theta^*-\theta^0))\right),
	\end{eqnarray} with learner-specific constant $\xi_{\eta;\theta^0}$ to the learner.
\end{prop}

Proposition \ref{costs} tells that teachers need to teach learners in accordance with their aptitude, and the optimal teaching example closely depends on learners themselves. Specifically, $(x^*,y^*)$ is picked based on $\theta^0$ and $\eta$ of learners. If a learner is well-equipped with \textit{primary knowledge} on $\theta^*$, i.e., its $\theta^0$ is very close to $\theta^*$, then $x^*$ of sampled $(x^*,y^*)$ should be very close to $\vec{0}$, that is, an indistinguishable one. But for a learner with poor \textit{primary knowledge}, a simple and distinguishable $x^*$ should be provided. This is consistent with teaching strategy in curriculum learning \cite{bengio2009curriculum, khan2011humans, lessard2019optimal}, where easy examples for learners at first, and difficult ones for learners as they approaches $\theta^*$. Moreover, for learners with different $\eta$, the magnitude of $x^*$ will be scaled relatively. For instance, an unresponsive learner with a small $\eta$ should be given a distinguishable $x^*$ in relatively large scale.

For pool-based teachers, the accessibility of the universal optimal teaching example in Theorem \ref{ost} should be considered additionally. We following present teacher-specific OSTS based on teacher-equipped operations.

\begin{prop}[OSTS for the combinable teacher]\label{comosts}
	For a learner with $\eta\neq0$ and $\theta^0$, the optimal teaching example $(x^*,y^*)$ for the combinable teacher with $\mathcal{P}_N=\mathcal{X}_N\times\mathcal{Y}_N$ is
	\begin{eqnarray}
		\begin{aligned}
			&x^*=\underset{x^\dagger\in\mathcal{X}_{\mathcal{P}}}{\arg\min}\quad \left\|x^\dagger-\xi_{\eta;\theta^0}(\theta^*-\theta^0)\right\|_2^2\\
			&\text{s.t.}\quad x^\dagger = \langle\beta,(x_1,\dots,x_N)^T\rangle,\qquad \beta\in\mathbb{R}^N,x_i\in\mathcal{X}_N
		\end{aligned}
	\end{eqnarray}
	where $\beta$ is called the generative coefficient vector, and $y^*=f^*(x^*)$.
\end{prop}
$\beta$ can be trained by a two-layer multilayer perceptron (MLP) illustrated in Figure \ref{nn} without nonlinear activation function, of which input and output are $(x_1,\dots,x_N)^T$ and $\langle\beta,(x_1,\dots,x_N)^T\rangle$, respectively. 

\begin{figure}[t]  
	\vskip -0.1in
	\subfigbottomskip=-6pt
	\centering\includegraphics[width=.5\columnwidth]{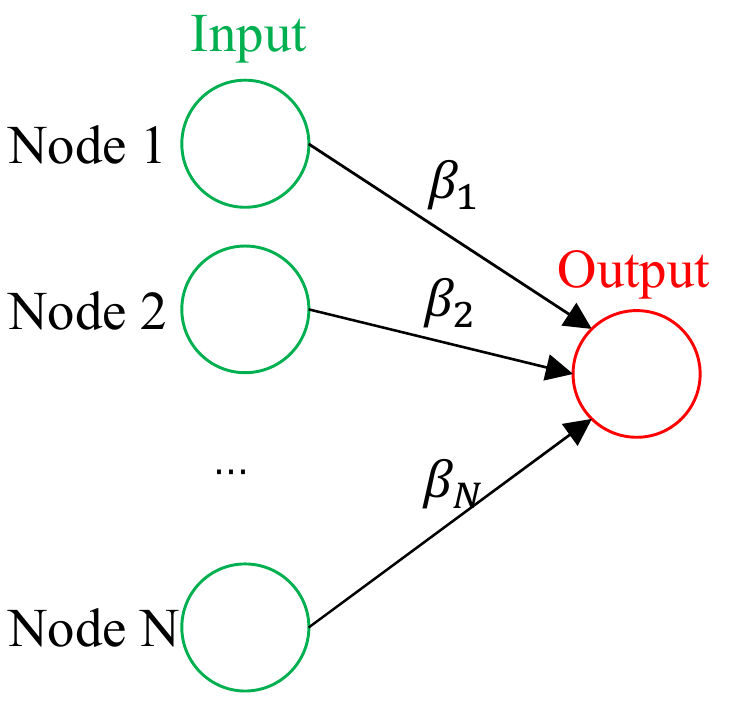}
	\vskip -0.1in
	\caption{The illustration of multilayer perceptron (MLP) structure. The input is feature vectors of $N$ teaching examples $x_1,\dots,x_N$ in pool $\mathcal{X}_N$. The connection weights are precisely generative coefficients. The output is the weighted sum of the input connections without nonlinear activation.}
	\label{nn}
	\vskip -0.2in
\end{figure}

\begin{algorithm}[th]
	\caption{One-shot teaching strategy (OSTS)}
	\label{aosts}
	{\bfseries Input:} Optimal $\theta^*$, initial $\theta^0$, learning rate $\eta$, small constant $\epsilon>0$ and maximal iteration number $T$. Small positive constant $\delta$ and $\mathcal{P}_N$ for pool-based teachers.
	\BlankLine
	\uIf{\rm teacher is complete}{
		Return $(x^*,y^*)=\left(\xi_{\eta;\theta^0}(\theta^*-\theta^0),f^*(\xi_{\eta;\theta^0}(\theta^*-\theta^0))\right)$ directly;}
	\uElseIf{\rm teacher is combinable}{
		Denote $\mathcal{L}(\beta)\coloneqq\left\|\langle\beta,(x_1,\dots,x_N)^T\rangle-\xi_{\eta;\theta^0}(\theta^*-\theta^0)\right\|_2^2,\qquad x_i\in\mathcal{X}_N$;
		
		Random initial $\beta^0$, set $\beta^t\leftarrow \beta^0$ and $t=0$;
		
		\While{$t\leq T$ {\rm and} $\mathcal{L}(\beta^t)\geq\epsilon$}{
			\BlankLine
			$\beta^t\leftarrow \beta^t-\delta\nabla_\beta\mathcal{L}(\beta^t)$;
			
			Set $t\leftarrow t+1$;
		}
		
		Return $(x^*,y^*)=\left(\langle\beta^t,(x_1,\dots,x_N)^T\rangle,f^*(\langle\beta^t,(x_1,\dots,x_N)^T\rangle)\right)$;
	}
	\uElseIf{\rm teacher is scalable}{
		Search $x^\dagger$ by
		$x^\dagger =\underset{x_i\in\mathcal{X}_N}{\arg\min}\quad\min\left(\left\|\frac{x_i}{\left\|x_i\right\|_2}-\frac{\theta^*-\theta^0}{\left\|\theta^*-\theta^0\right\|_2}\right\|_2^2,\left\|\frac{x_i}{\left\|x_i\right\|_2}+\frac{\theta^*-\theta^0}{\left\|\theta^*-\theta^0\right\|_2}\right\|_2^2\right)$;
		
		Return $x^* =\kappa x^\dagger$ and $y^*=f^*(x^*)$, in which $\kappa = \sgn\left(\left\|\frac{x_i}{\left\|x_i\right\|_2}+\frac{\theta^*-\theta^0}{\left\|\theta^*-\theta^0\right\|_2}\right\|_2^2-\left\|\frac{x_i}{\left\|x_i\right\|_2}-\frac{\theta^*-\theta^0}{\left\|\theta^*-\theta^0\right\|_2}\right\|_2^2\right)\cdot\xi_{\eta;\theta^0}\left\|\theta^*-\theta^0\right\|_2 /\left\|x^\dagger\right\|_2
		$;
	}
	\ElseIf{\rm teacher is naive}{
		Return $x^*=\underset{x_i\in\mathcal{X}_N}{\arg\min}\quad \left\|x_i-\xi_{\eta;\theta^0}(\theta^*-\theta^0)\right\|_2^2$
		and $y^*=f^*(x^*)$.
	}
\end{algorithm}

\begin{prop}[OSTS for the scalable teacher]\label{sosts}
	For a learner with $\eta\neq0$ and $\theta^0$, the optimal teaching example $(x^*,y^*)$ for the scalable teacher with $\mathcal{P}_N=\mathcal{X}_N\times\mathcal{Y}_N$ is 
	\begin{eqnarray}
		x^* =\kappa x^\dagger 
	\end{eqnarray}
	 and $y^*=f^*(x^*)$, in which $\kappa = \sgn\left(\left\|\frac{x_i}{\left\|x_i\right\|_2}+\frac{\theta^*-\theta^0}{\left\|\theta^*-\theta^0\right\|_2}\right\|_2^2-\left\|\frac{x_i}{\left\|x_i\right\|_2}-\frac{\theta^*-\theta^0}{\left\|\theta^*-\theta^0\right\|_2}\right\|_2^2\right)\cdot\xi_{\eta;\theta^0}\left\|\theta^*-\theta^0\right\|_2 /\left\|x^\dagger\right\|_2
	 $ and 
	 \begin{eqnarray}
	 	x^\dagger =\underset{x_i\in\mathcal{X}_N}{\arg\min}\quad\min\left(\left\|\frac{x_i}{\left\|x_i\right\|_2}-\frac{\theta^*-\theta^0}{\left\|\theta^*-\theta^0\right\|_2}\right\|_2^2,\left\|\frac{x_i}{\left\|x_i\right\|_2}+\frac{\theta^*-\theta^0}{\left\|\theta^*-\theta^0\right\|_2}\right\|_2^2\right)
	 \end{eqnarray}
\end{prop}

For the scalable teacher, the fundamental difference among examples is their normalized vector derived by operator $\frac{\cdot}{\left\|\cdot\right\|_2}$ since the magnitude can be scaled correspondingly. Besides, the operator $\min\left(\left\|\cdot-\cdot\right\|_2^2,\left\|\cdot+\cdot\right\|_2^2\right)$ is to avoid the influence of opposite direction between $x_i$ and $\theta^*-\theta^0$. Therefore, the scalable teacher should pick the example whose
normalized vector is close to that of the universal optimal teaching example when ignoring the direction. For the naive teacher, it is straightforward to derive associated OSTS.

\begin{prop}[OSTS for the naive teacher]\label{nosts}
	For a learner with $\eta\neq0$ and $\theta^0$, the optimal teaching example $(x^*,y^*)$ for the naive teacher with $\mathcal{P}_N=\mathcal{X}_N\times\mathcal{Y}_N$ is
	\begin{eqnarray}
		&x^*=\underset{x_i\in\mathcal{X}_N}{\arg\min}\quad \left\|x_i-\xi_{\eta;\theta^0}(\theta^*-\theta^0)\right\|_2^2,
	\end{eqnarray}
	and $y^*=f^*(x^*)$.
\end{prop}

One can observe that pool-based teachers need more workload to construct the optimal teaching example. Specifically, different from the complete teacher, OSTS requires scalable (combinable) teachers to do further scaling (linear combination) to match the universal optimal teaching example Eq.(\ref{eq9}). Packing OSTS for the complete and pool-based teachers up, the pseudo code is summarized in Alg.\ref{aosts}.

Limited knowledge knowledge leads to constrained capability of teachers. Therefore, the condition of learners is discussed as follows, satisfying which they can converge after one iteration with the help of a pool-based teacher. In another word, pool-based teachers can construct the universal optimal teaching example for the learner who satisfies such condition.

\begin{thm}\label{snc}
	For a learner with $\eta\neq0$ and $\theta^0$, the pool-based teachers are able to construct the universal optimal teaching example, if $\theta^0$ satisfies following teacher-specific necessary and sufficient condition.
	\begin{enumerate}[label= {\rm (\alph*)}]
		\item $\theta^0-\theta^*\in\text{\rm span}(\mathcal{X}_N)$ for the combinable teacher;
		
		\item $\theta^0-\theta^*\in \alpha\cdot\mathcal{X}_N, \alpha\in\mathbb{R}$ for the scalable teacher;
		
		\item $\xi_{\eta;\theta^0}(\theta^0-\theta^*)\in \mathcal{X}_N$ for the naive teacher.
	\end{enumerate}
\end{thm}
\begin{proof} The proof of the naive teacher (c) is trivial so we omit it.
	
"$\Rightarrow$":

	\begin{enumerate}[label= {\rm (\alph*)}]
		\item Since $\theta^0-\theta^*\in\text{\rm span}(\mathcal{X}_N)$, we can derive $\exists\, \lambda_1,\lambda_2,\dots,\lambda_N \in\mathbb{R}$ not all zero, such that
		\begin{eqnarray}
			\theta^0-\theta^*=\sum_{i=1}^{N}\lambda_i x_i, \qquad x_i\in\mathcal{X}_N.
		\end{eqnarray} Hence, we can obtain $\xi_{\eta;\theta^0}(\theta^*-\theta^0)=-\xi_{\eta;\theta^0}\sum_{i=1}^{N}\lambda_i x_i$. Let $\beta_i=-\xi_{\eta;\theta^0}\lambda_i$, then the combinable teacher can derive the universal optimal teaching example.
		
		\item Similarly, $\because\theta^0-\theta^*\in \alpha\cdot\mathcal{X}_N,\,\therefore\exists\alpha\neq0\in\mathbb{R},x_i\in\mathcal{X}_N$ such that $\theta^0-\theta^*=\alpha x_i$. Then the scalable teacher can obtain the universal optimal teaching example by setting $\beta=-\xi_{\eta;\theta^0}\alpha\cdot e_i$.
	\end{enumerate}

"$\Leftarrow$":

	\begin{enumerate}[label= {\rm (\alph*)}]
		\item Since the combinable teacher can derive the universal optimal teaching example, we can derive $\exists\, \beta\in\mathbb{R}^n$, such that
		\begin{eqnarray}
			\xi_{\eta;\theta^0}(\theta^*-\theta^0)=\langle\beta,(x_1,\dots,x_N)^T\rangle, \qquad x_1,\dots,x_N\in\mathcal{X}_N.
		\end{eqnarray}
		Since the constant $\xi_{\eta;\theta^0}$ is non-zero, we have
		\begin{eqnarray}
			\theta^0-\theta^*=\sum_{i=1}^{N}-\frac{\beta_i}{\xi_{\eta;\theta^0}} x_i, \qquad x_i\in\mathcal{X}_N, 
		\end{eqnarray} 
		from which we can conclude $\theta^0-\theta^*\in\text{\rm span}(\mathcal{X}_N)$.

		\item Analogously, we can derive $\exists\, \beta\in\mathbb{R}$, such that \begin{eqnarray}
			\xi_{\eta;\theta^0}(\theta^*-\theta^0)=\beta\cdot x_i, \qquad x_i\in\mathcal{X}_N.
		\end{eqnarray}
		For a non-zero constant $\xi_{\eta;\theta^0}$, we have 
		\begin{eqnarray}
			\theta^0-\theta^*=-\frac{\beta}{\xi_{\eta;\theta^0}} x_i,
		\end{eqnarray} 
		i.e., $\theta^0-\theta^*\in \alpha\cdot\mathcal{X}_N, \alpha\in\mathbb{R}$.
	\end{enumerate}
\end{proof}

Theorem \ref{snc} tells that the teaching ability of pool-based teachers depends on equipped operation. Specifically, the combinable teacher is better qualified than other pool-based teachers, who can construct the universal optimal teaching example if the difference between $\theta^0$ of these learners and $\theta^*$ is expressible in terms of $\mathcal{X}_N$. For the scalable teacher, the capability is worse. They are only able to generate the universal optimal teaching example for those learners whose $\theta^0$ satisfy $\theta^0-\theta^*\in\alpha\mathcal{X}_N$. Moreover, for the naive teacher, they only can do that when the universal optimal teaching example is in their pools.

\begin{rem}
	The specific task for model $\theta^*$ can be not only regression but also classification. Besides above analysis of OSTS is mainly on binary classification, based on which K-class cases can be extend straightforward since a K-class classification task usually can be separated into K binary classification tasks.
\end{rem}

\subsection{Learner loss functions}\label{gl}

We mainly discuss three typical loss functions (square loss, hinge loss and logistic loss) under the complete teacher, which can be extended to the pool-based teacher following Proposition \ref{comosts}, \ref{sosts}, \ref{nosts}. The learners with these loss functions are called least square regression (LSR), support vector machine (SVM) and logistic regression (LR) learners, respectively. Without loss of generality, we consider the worse case where the learner has no primary knowledge, i.e., $\theta^0=0$.

\begin{prop}\label{prop9}
	Square loss, $\ell(f, y)=\left(f - y\right)^2$ is one-shot teachable for the complete teacher.
\end{prop}

\begin{proof}
	It follows from Theorem \ref{ost} that the first-order differentiable learner loss is one-shot teachable, and it is evident that square loss $\left(f - y\right)^2$ has first-order partial derivatives w.r.t. $f$, which justifies the existence of the optimal teaching example. Hence, to prove square loss is one-shot teachable, it needs to solve for $\xi_{\eta;\theta^0}$.
	
	Following Proposition \ref{epg}, we have 
	\begin{eqnarray*}
		\mathcal{G}\left(\theta^0;(x,y)|\left(f - y\right)^2;\left\langle \theta,x\right\rangle\right)=2\left(\left\langle \theta^0,x\right\rangle-y\right)x.
	\end{eqnarray*}
	Thus, we obtain 
	\begin{eqnarray*}
		\xi_{\eta;\theta^0}=1/(2\eta y).
	\end{eqnarray*}
	Therefore, the optimal teaching example $(x^*,y^*)$ is 
	\begin{equation*}
		(x^*,y^*)=\left((\theta^*-\theta^0)/(2\eta y^*),y^*\right)
	\end{equation*}
	where $y^*$ is a free variable, which can be arbitrarily chosen in $\mathcal{Y}$ at first.
\end{proof}

We see that for the dynamic learner equipped with variable learning rate $\eta^t, t=0,1,\dots,T$, the teacher merely requires the initial $\eta^0$ to construct the optimal teaching example. For hinge loss, we mainly consider homogeneous $\theta$ (without a bias term), and the inhomogeneous result can be derived similarly. 

\begin{prop}\label{prop10}
	Hinge loss, $\ell(f,y)=\max\left(1-y\cdot f,0\right)$ is one-shot teachable for the complete teacher.
\end{prop}

\begin{proof}
	The hinge loss, $\text{max}\left(1-y\cdot f,0\right)$, is also first-order differentiable w.r.t. $f$. Besides, $(x^*,y^*)$ should be one of the teaching examples $(x,y)$ who satisfy $1-y\cdot f(\theta^0,x)>0$, otherwise the gradient at $(x^*,y^*)$ is zero and $\theta^0$ will not be updated. Therefore, we have 
	\begin{eqnarray*}
		\mathcal{G}\left(\theta^0;(x,y)|1-y\cdot f;\left\langle \theta,x\right\rangle\right)=-yx.
	\end{eqnarray*}
	Consequently, we have
	\begin{eqnarray*}
		\xi_{\eta;\theta^0}=1/(\eta y),
	\end{eqnarray*}
	which yields the optimal teaching example
	\begin{equation*}
		(x^*,y^*)=\left((\theta^*-\theta^0)/(\eta y^*),y^*\right)
	\end{equation*}
	where $y^*$ is a free variable as well. This completes the proof.
\end{proof}

\begin{prop}\label{prop12}
	Logistic loss, $\ell(f,y)=\log\left(1+\exp\left(-y\cdot f\right)\right)$ is one-shot teachable for the complete teacher.
\end{prop}

\begin{proof}
	The logistic loss, $\text{log}\left(1+\text{exp}\left(-y\cdot f\right)\right)$, is first-order differentiable w.r.t. $f$ as well. 
	It follows from Proposition \ref{epg} that
	\begin{eqnarray*}
		&&\mathcal{G}\left(\theta^0;(x,y)|\text{log}\left(1+\text{exp}\left(-y\cdot f\right)\right);\left\langle \theta,x\right\rangle\right)\\
		&=&\frac{-y\text{exp}\left(-y\left\langle \theta^0,x\right\rangle\right)}{1+\text{exp}\left(-y\left\langle \theta^0,x\right\rangle\right)}x=\frac{-y}{1+\text{exp}\left(y\left\langle \theta^0,x\right\rangle\right)}x.
	\end{eqnarray*}
	Thus, we have 
	\begin{eqnarray*}
		\xi_{\eta;\theta^0}=2/(\eta y).
	\end{eqnarray*}
	It derives the optimal teaching example
	\begin{equation*}
		(x^*,y^*)=\left(2(\theta^*-\theta^0)/(\eta y^*),y^*\right)
	\end{equation*}
	with free variable $y^*$, which is the desired result.
\end{proof}

As expected, these first-order differentiable loss functions are one-shot teachable, which justifies that one-shot teachability is not a strict condition and demonstrates the general applicability of OSTS.

\section{Experiments}\label{exmp}

In this section, we compare OSTS with machine learning and iterative strategies on synthetic and real-world data, which demonstrates the efficiency of OSTS.

Specifically, we derive $\theta^*$ by predefining (synthetic data) or pretraining (real-world dataset). L-2 norm is set to measure learner bias $\mathcal{B}(\theta^t,\theta^*)=\|\theta^t-\theta^*\|_2$. Without
particular emphasis, experiments are implemented under
the complete teachers setting.
We teach linear model binary classification tasks and nonlinear neural networks (convolutional neural network) 10-class classification tasks.

\subsection{Comparing OSTS with machine learning strategies}

\begin{figure}[t]
	\vskip -0.1in
	\subfigbottomskip=-6pt
	\subfigcapskip=0pt
	\centering
	\subfigure[MNIST 0 VS. 1]{\includegraphics[width = 0.32\linewidth]{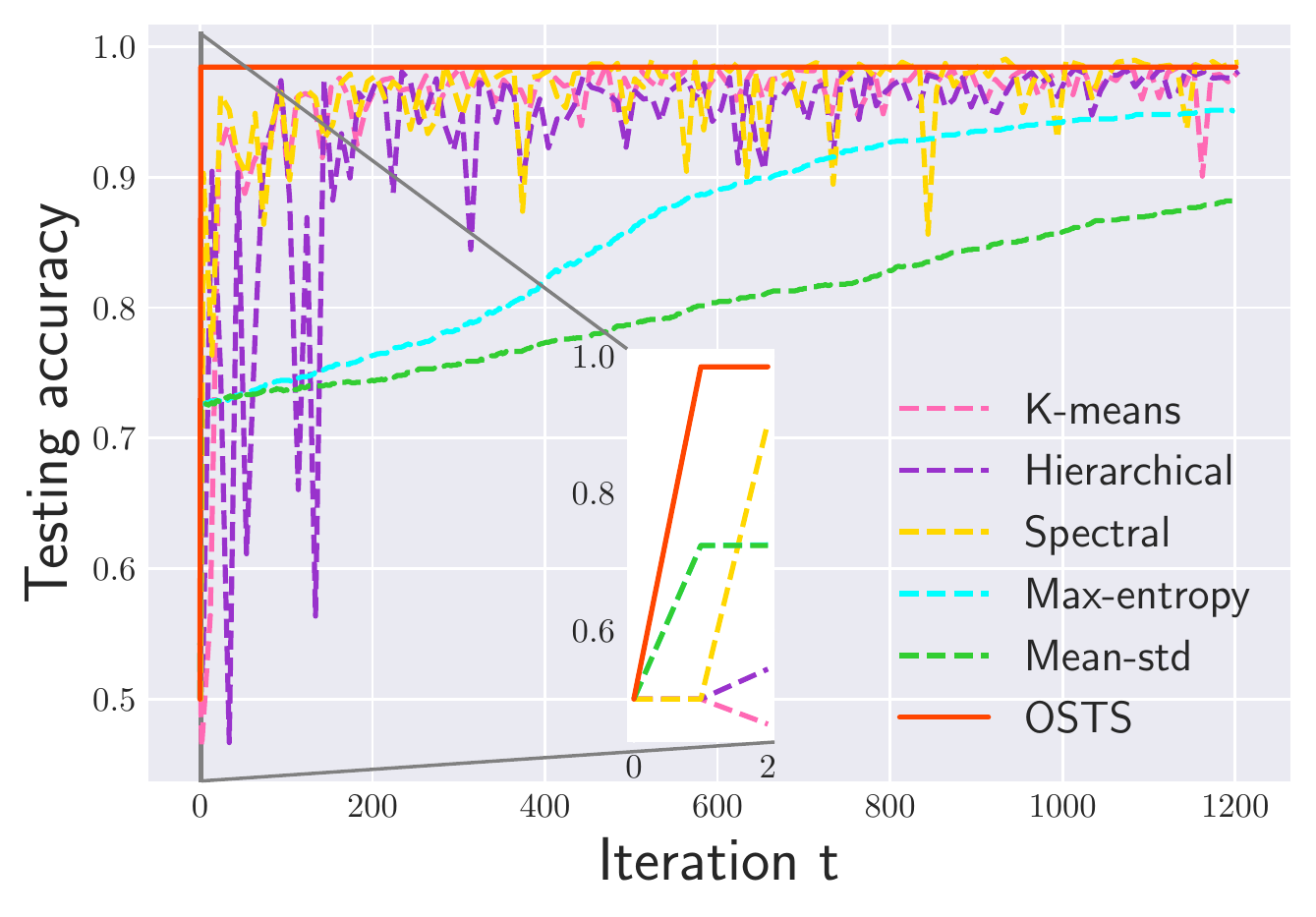}}
	\subfigure[MNIST 3 VS. 8]{\includegraphics[width = 0.32\linewidth]{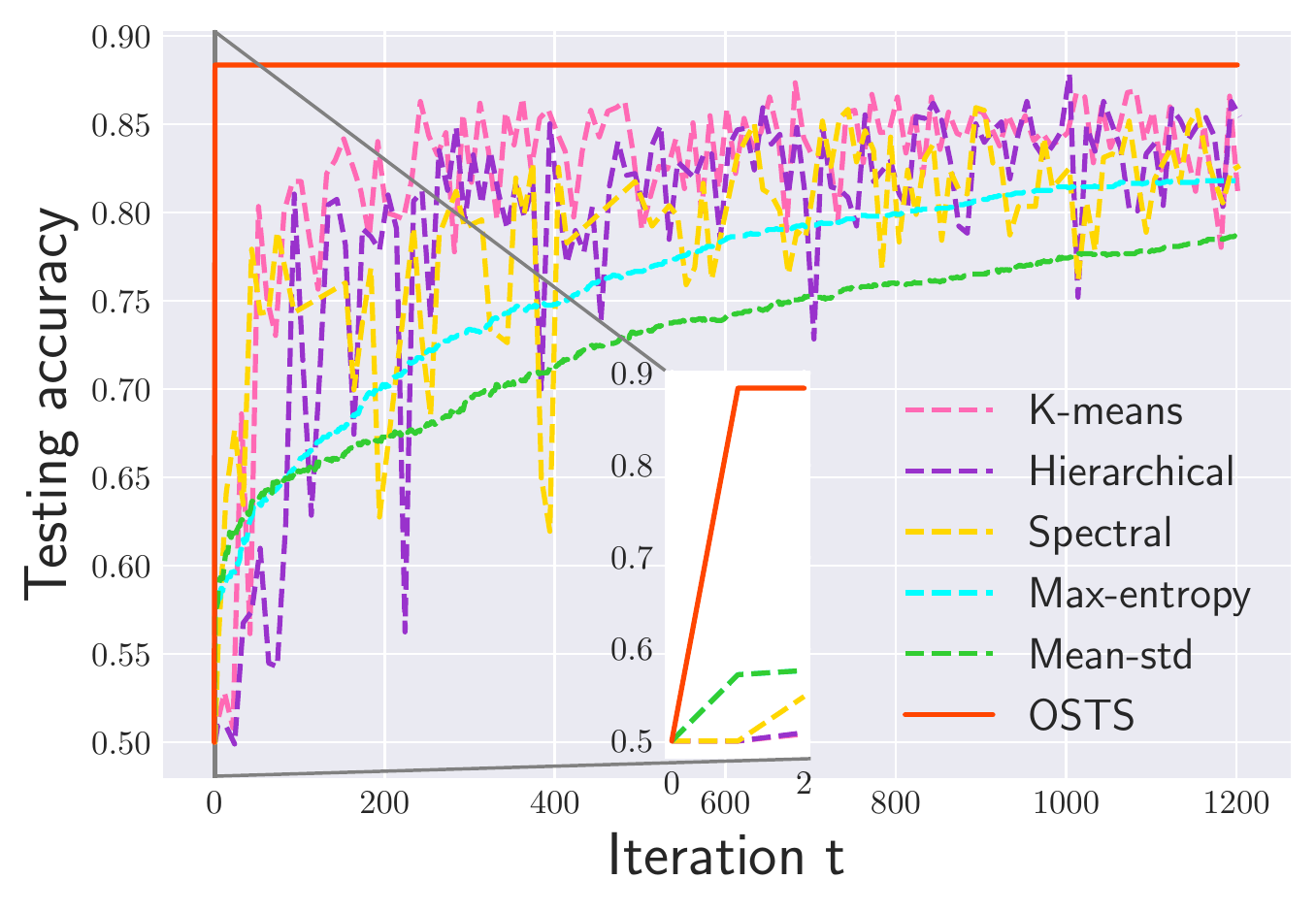}}
	\subfigure[CIFAR10]{\includegraphics[width = 0.32\linewidth]{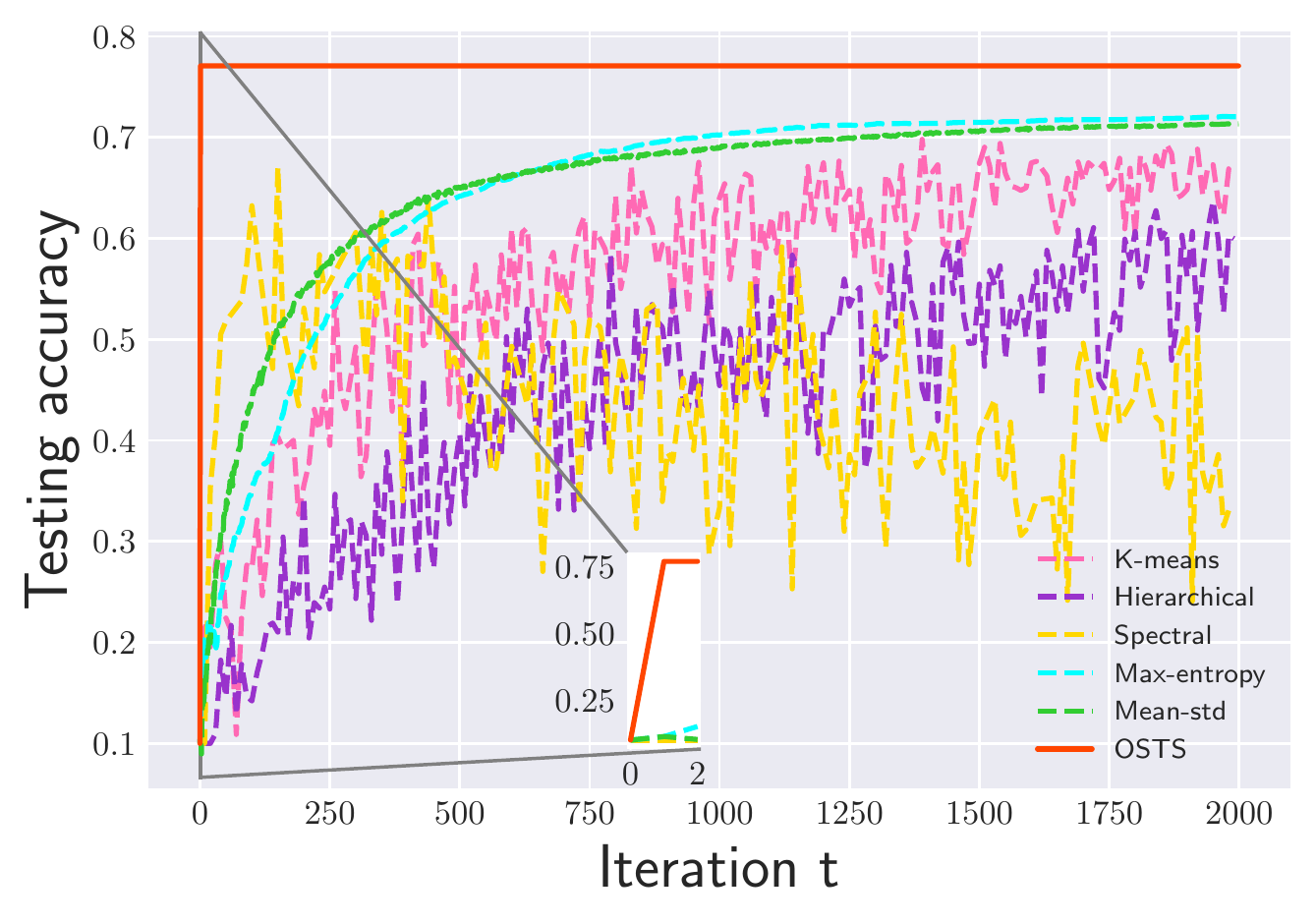}}
	\caption{OSTS vs. machine learning strategies. LR learners following OSTS converge faster and have higher testing accuracy than that guided by other five machine learning strategies on both binary and 10-class classification tasks.}
	\label{ml}
	\vskip -0.2in
\end{figure}

\begin{figure*}[t]
	\vskip -0.1in
	\subfigbottomskip=-6pt
	\subfigcapskip=0pt
	\subfigure[LSR leaner]{
		\begin{minipage}[c]{\linewidth}
			\includegraphics[width=0.32\linewidth]{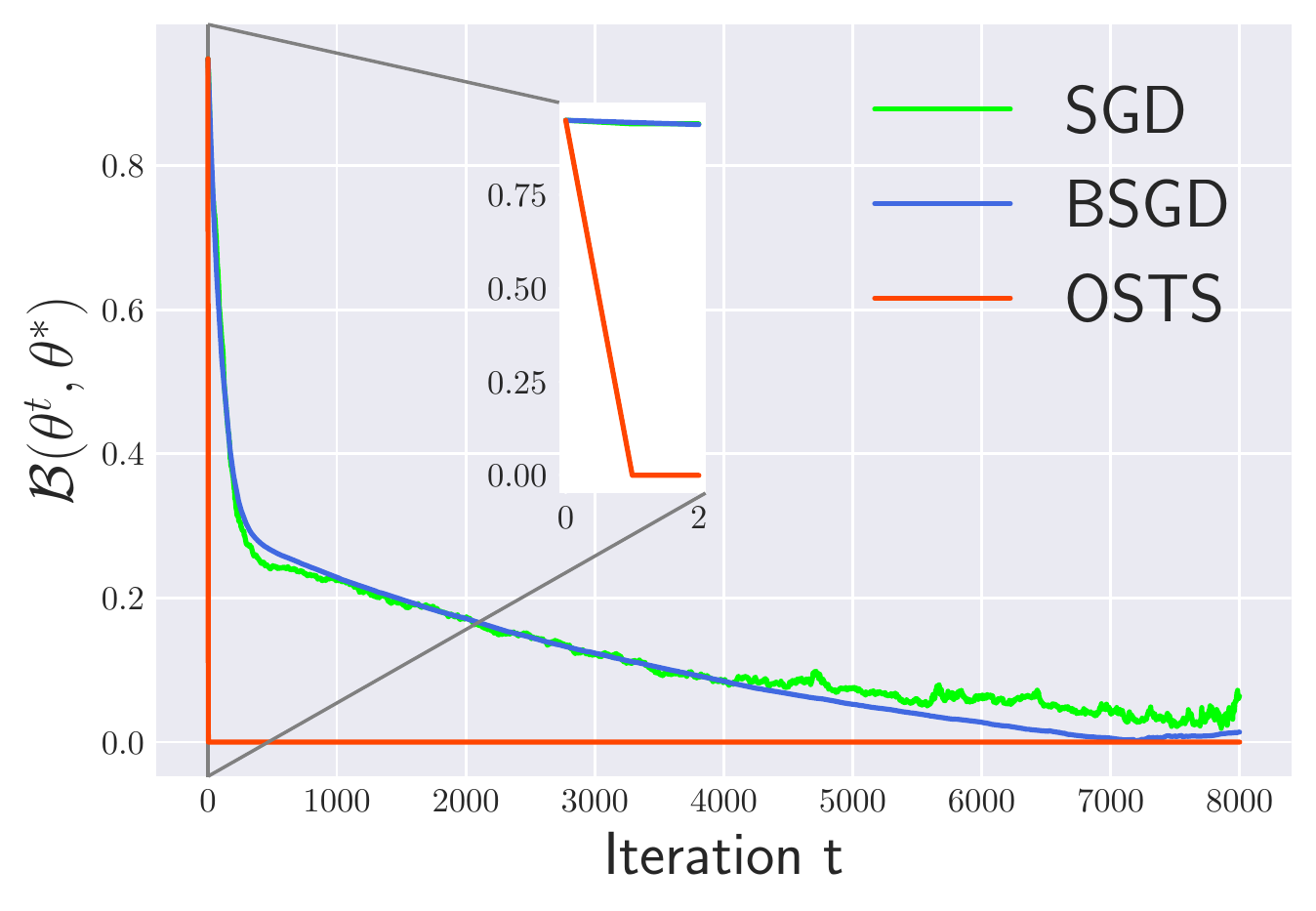}
			\includegraphics[width=0.32\linewidth]{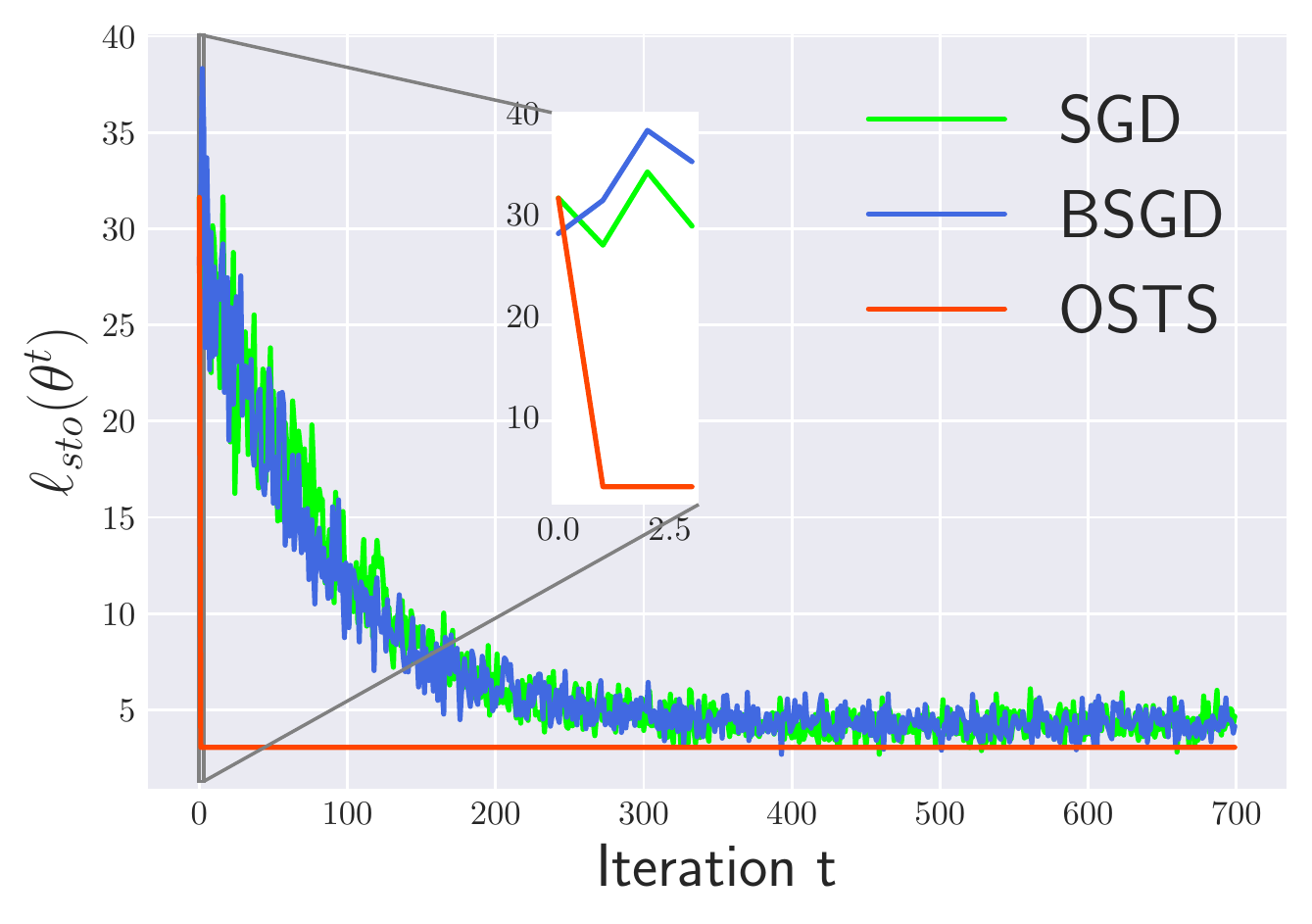}
			\includegraphics[width=0.32\linewidth]{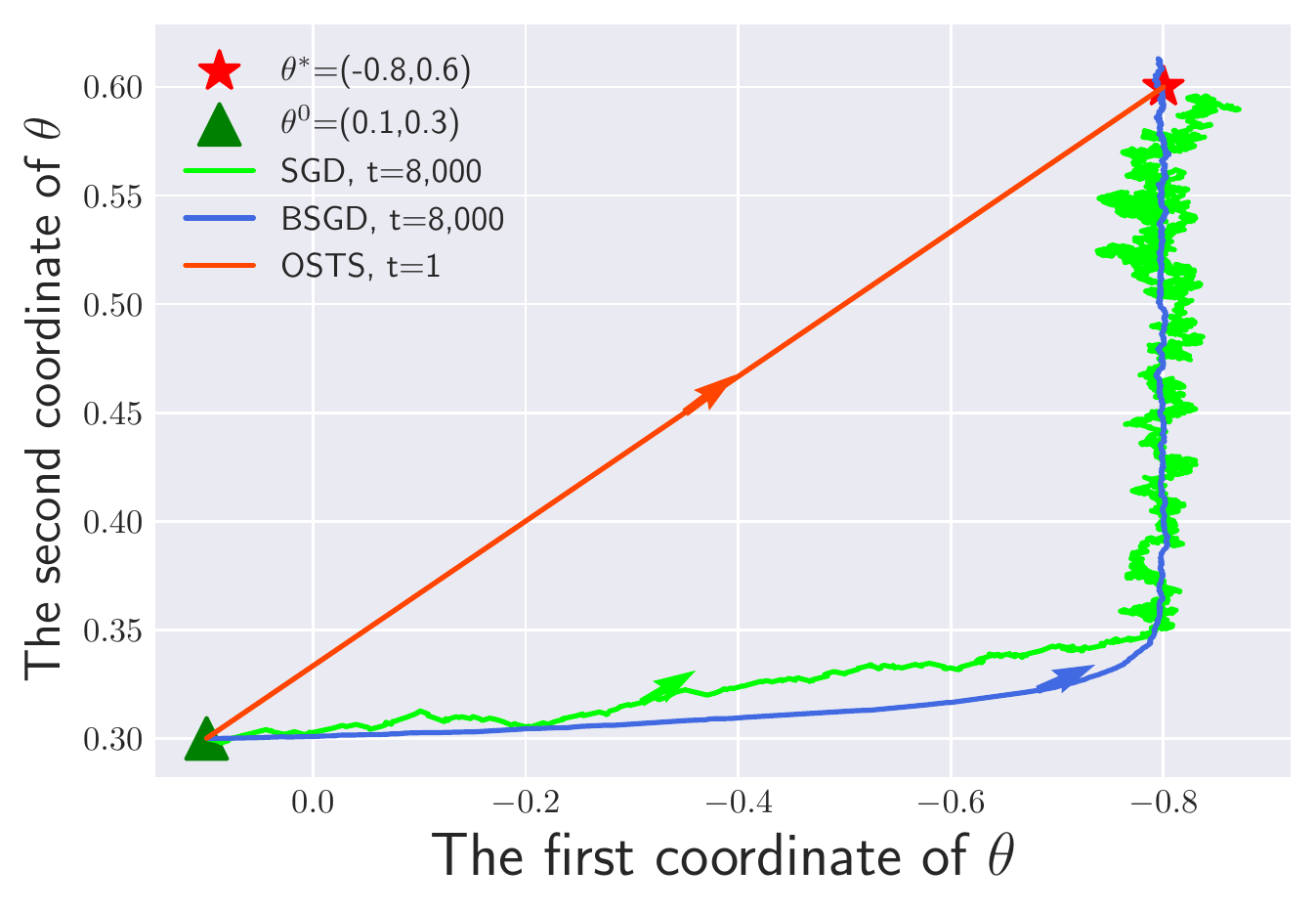}
		\end{minipage}
	}\\
	\subfigure[SVM leaner]{
		\begin{minipage}[c]{\linewidth}
			\includegraphics[width=0.32\linewidth]{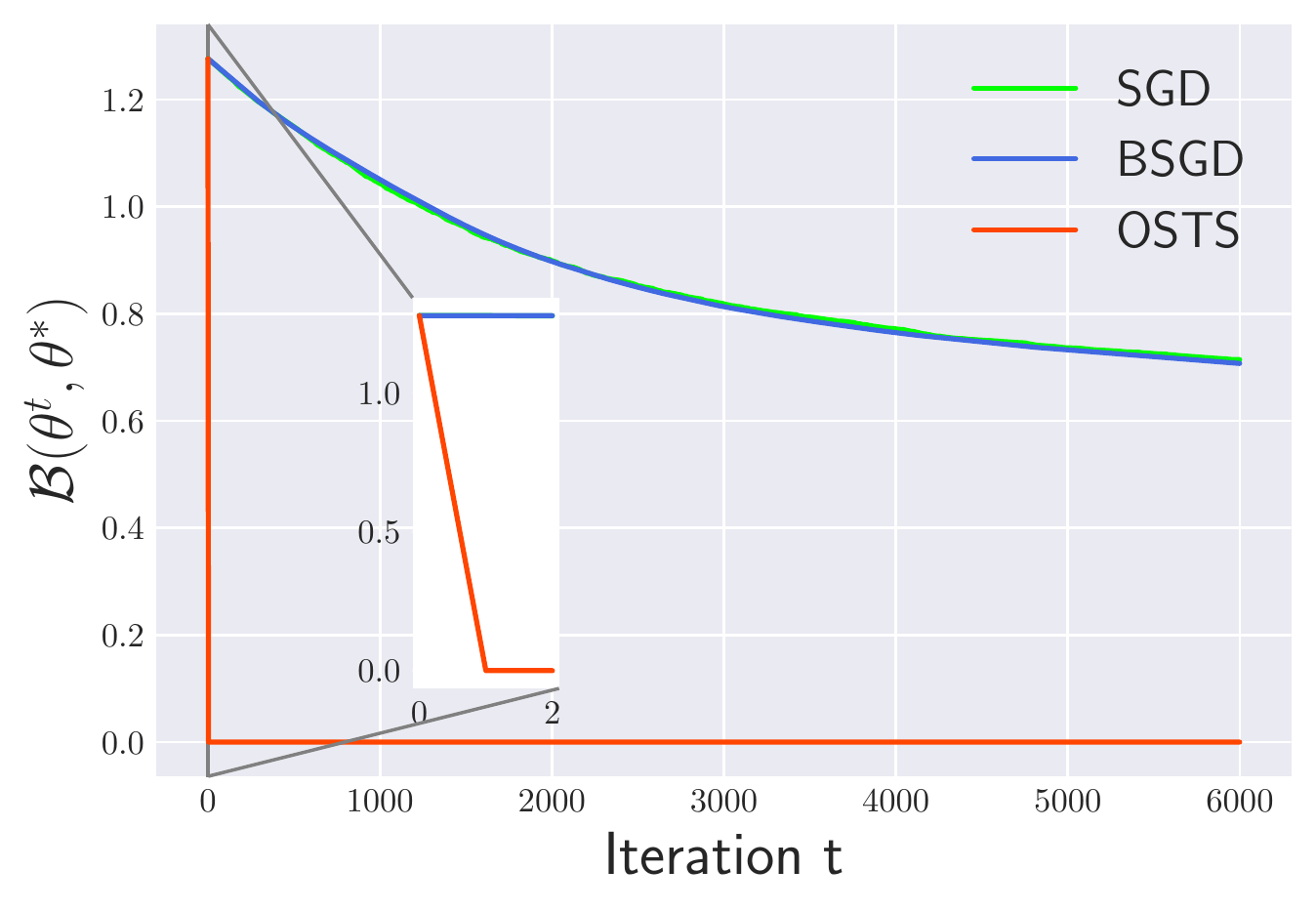}
			\includegraphics[width=0.32\linewidth]{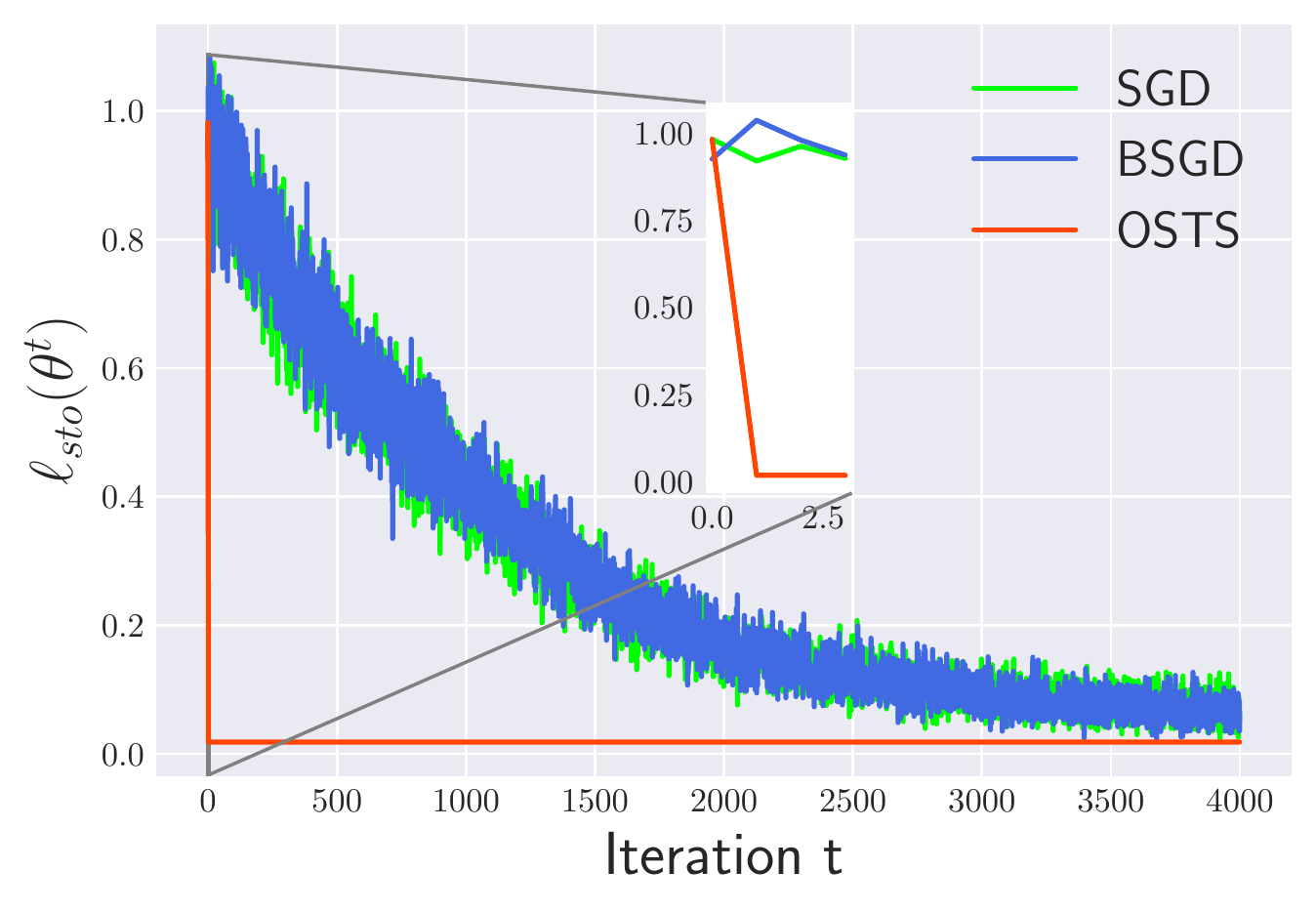}
			\includegraphics[width=0.3\linewidth]{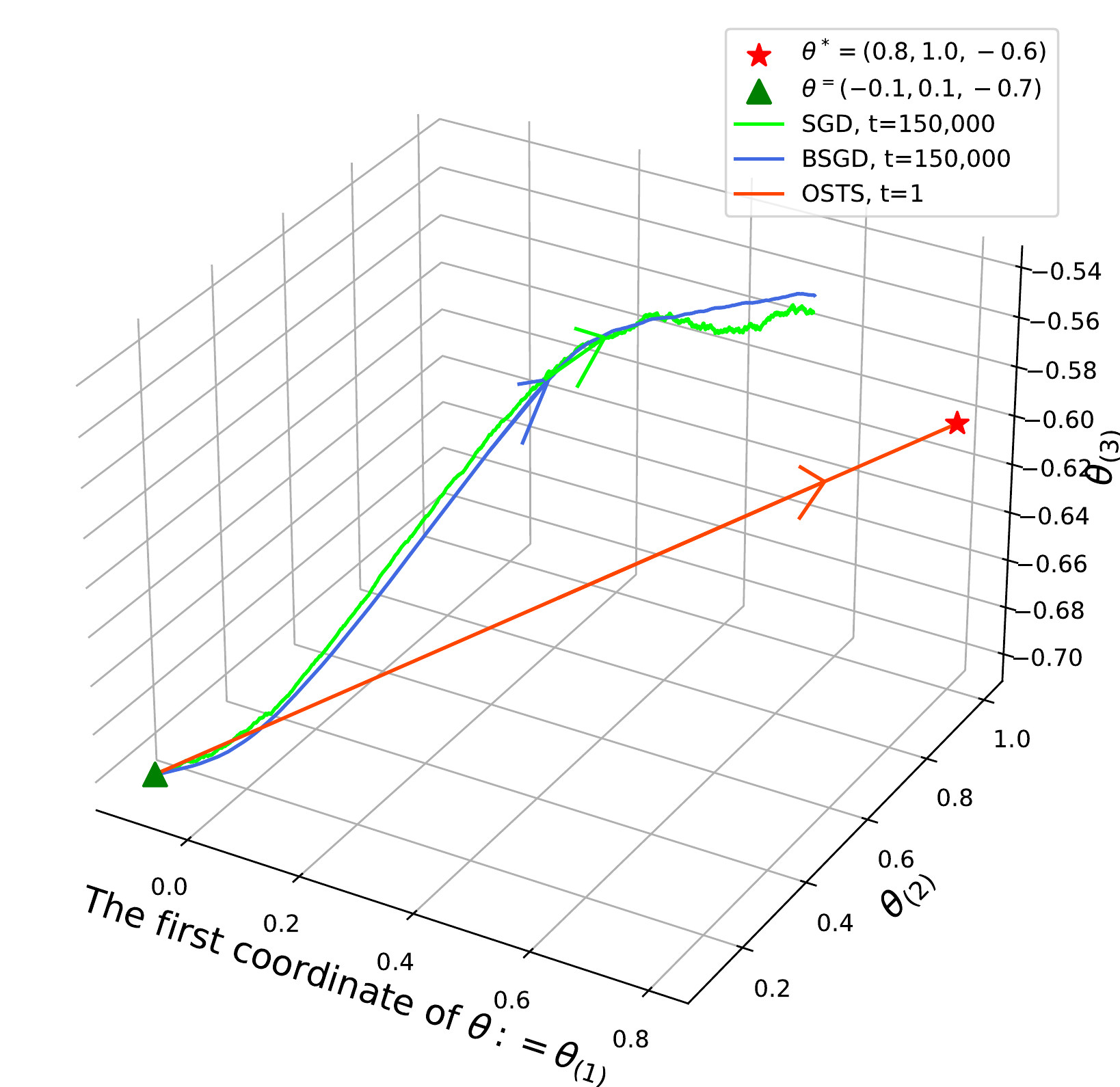}
		\end{minipage}
	}\\
	\subfigure[LR leaner]{
		\begin{minipage}[c]{\linewidth}
			\includegraphics[width=0.32\linewidth]{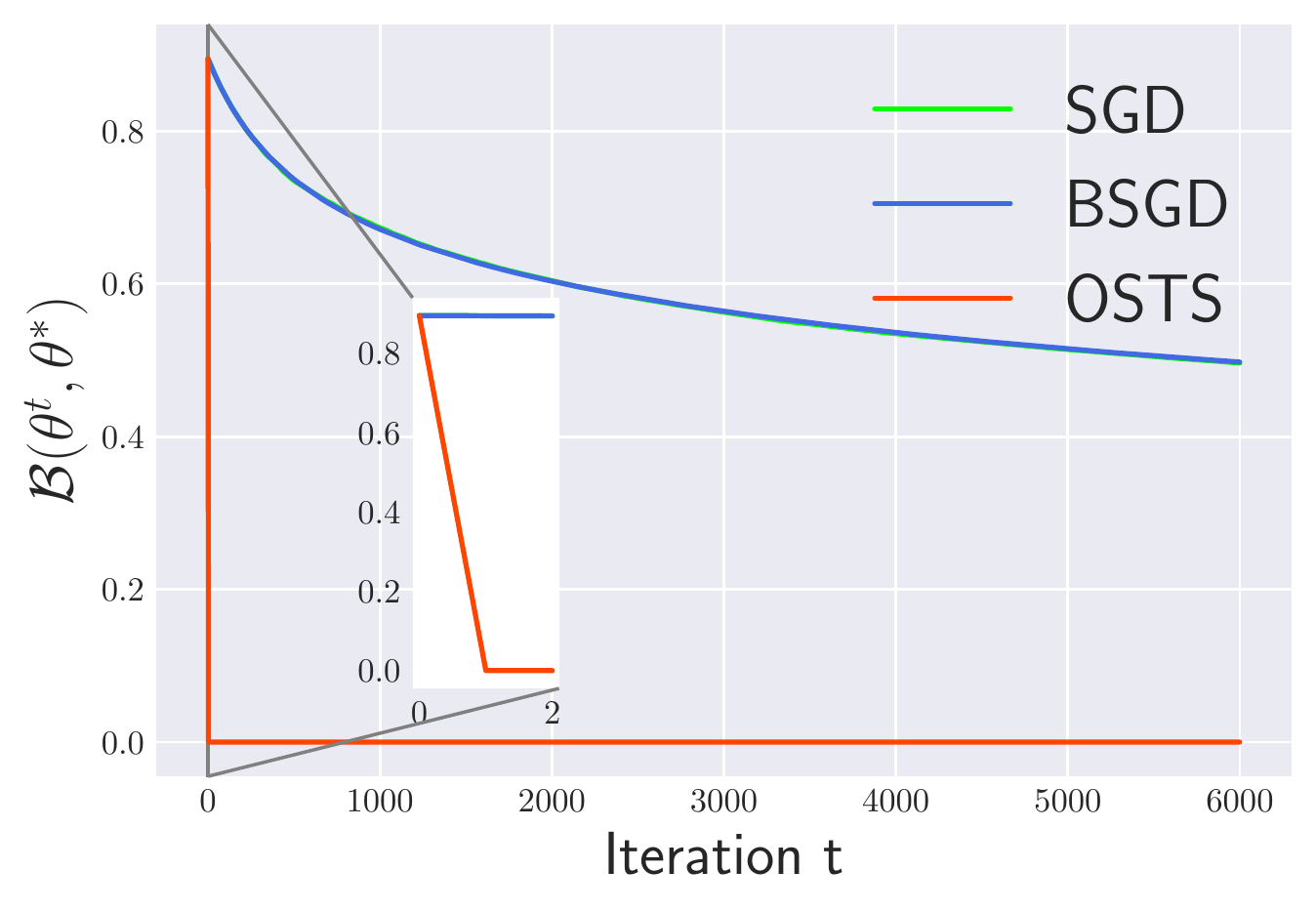}
			\includegraphics[width=0.32\linewidth]{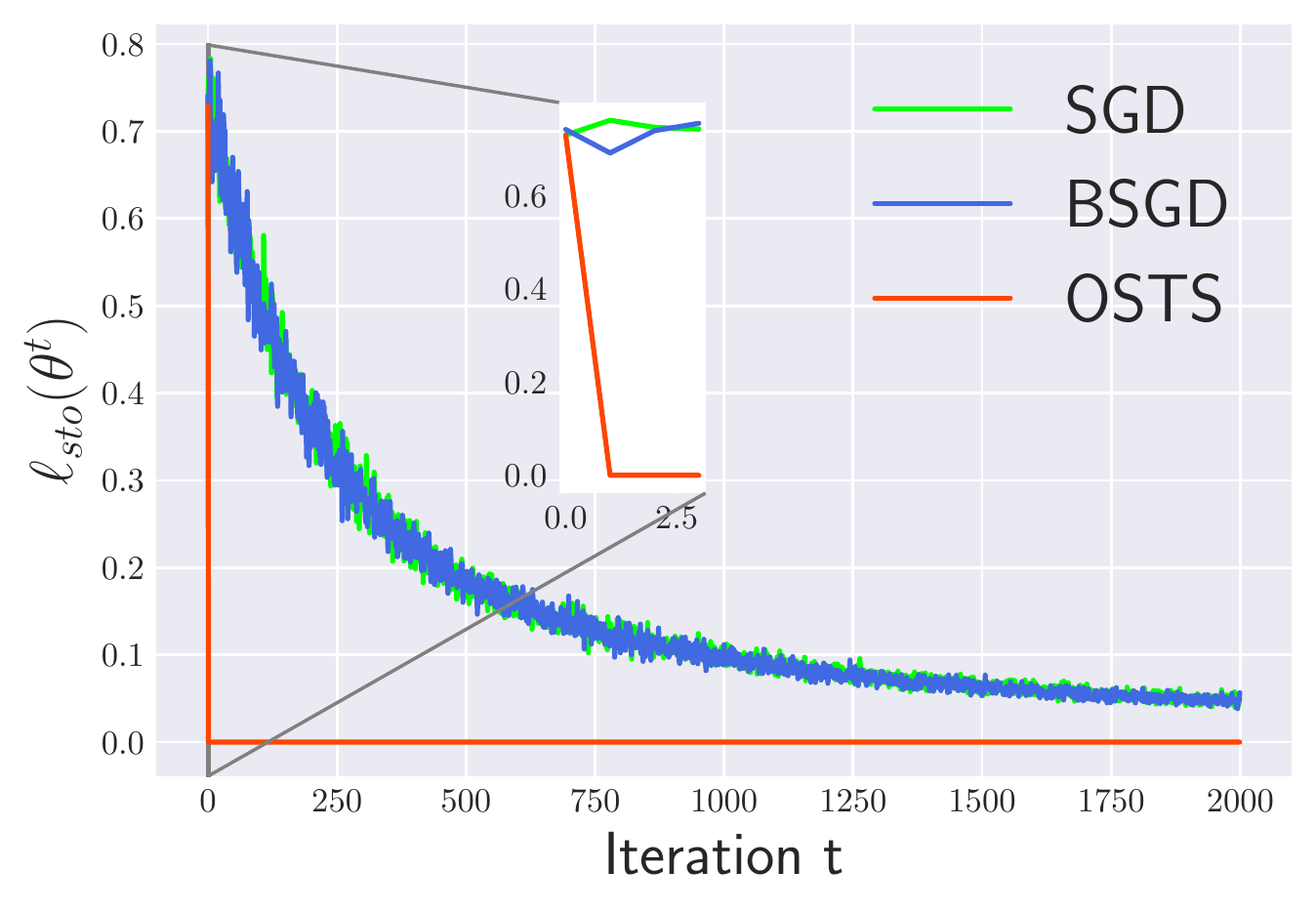}
			\includegraphics[width=0.3\linewidth]{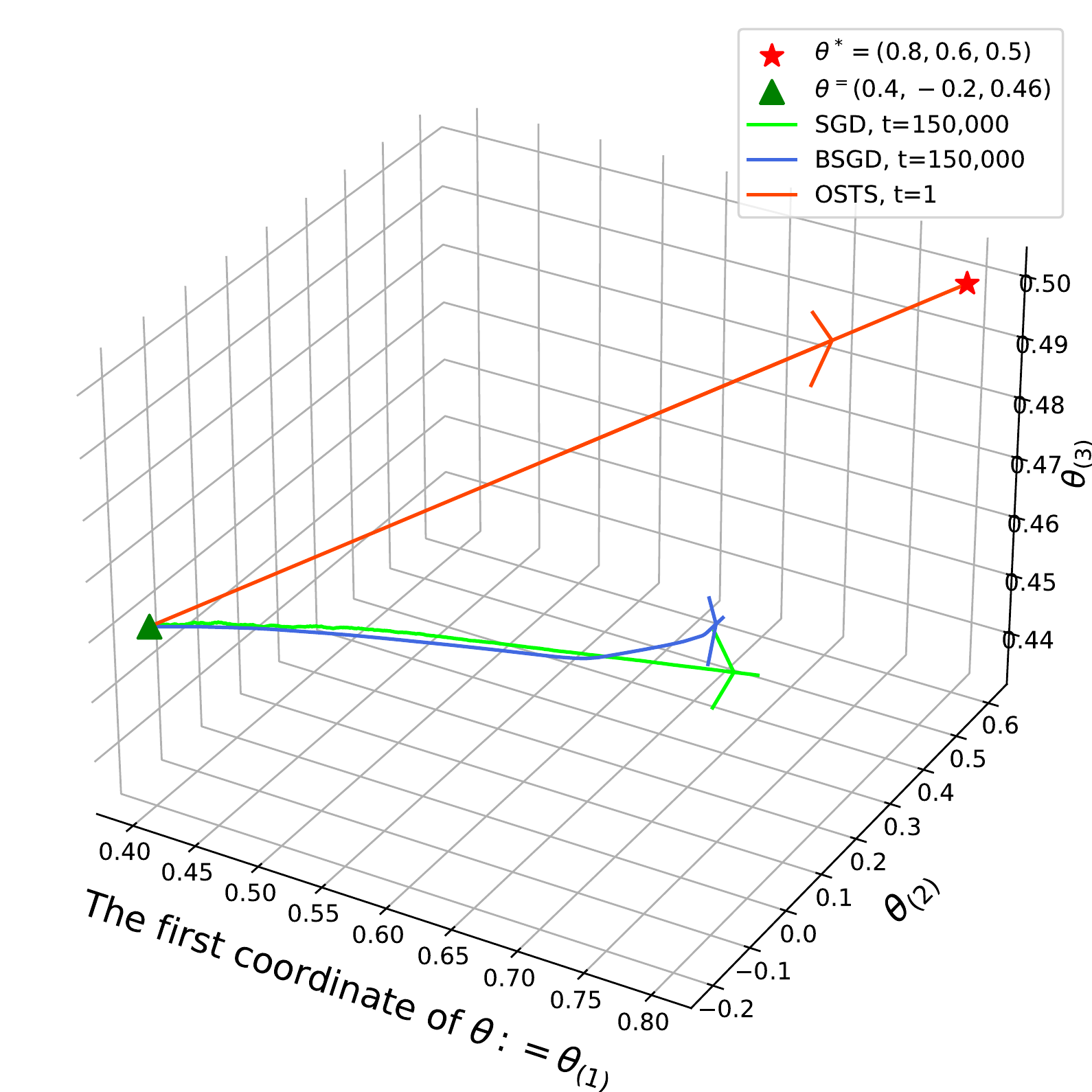}
		\end{minipage}
	}
	\caption{Teaching three typical learners on Gaussian data. Figures of $\mathcal{B}(\theta^t,\theta^*)$ and $\theta^t$ evolution of three learners show that OSTS arrives $\theta^*$ at one step while SGD and BSGD approach $\theta^*$ gradually. $\ell_{sto}(\theta)$ of OSTS drops significantly faster than that of SGD and BSGD.}
	\label{b_l_lp}
	\vskip -0.2in
\end{figure*}

We compare testing accuracy of LR learners guided by OSTS and machine learning strategies on MNIST \cite{lecun1998mnist} and CIFAR10 \cite{cifar10}. In detail, unsupervised baselines are K-means, hierarchical and spectral clustering, and supervised ones are active learning strategies including max-entropy maximizing the predictive entropy \cite{shannon2001mathematical} and mean-std maximizing the mean standard deviation \cite{kampffmeyer2016semantic,kendall2018segnet}. As expected, using one teaching example, OSTS surpass all these baselines. Specifically, OSTS has the highest testing accuracy, and it also costs less examples than other strategies to help learners converge to $\theta^*$.

Figure \ref{ml} shows that no matter binary or 10-class classification tasks, OSTS is the most efficient strategies in terms of the cost examples and the testing accuracy. This illustrates the necessity of teaching, that is, being taught can be faster than asking questions \cite{rivest1995being}. Further, such outperformance is increasingly significant as the complexity of classification tasks grows, which implies the higher value of OSTS in complicated cases. On the other hand, active learning strategies outperform unsupervised ones in complicated cases since they have more clear objective of example selection, towards which it can accelerate improvement of accuracy. Since we sample randomly for clustering, the curves of unsupervised strategies are winding.

\subsection{Comparing OSTS with iterative strategies}

When comparing OSTS with iterative strategies (SGD and BSGD), we further assess the performance including learner bias $\mathcal{B}(\theta^t,\theta^*)$ and testing loss $\ell(\theta^t)$. As expected, costing less, OSTS outperforms iterative baselines in terms of the speed of convergence and the testing accuracy on synthetic and real-world (MNIST and CIFAR10) data.

\begin{figure}[t]
	\vskip -0.1in
	\subfigbottomskip=-6pt
	\subfigcapskip=0pt
	\subfigure[0 VS. 1]{
		\begin{minipage}[c]{\linewidth}
			\includegraphics[width = 0.32\linewidth]{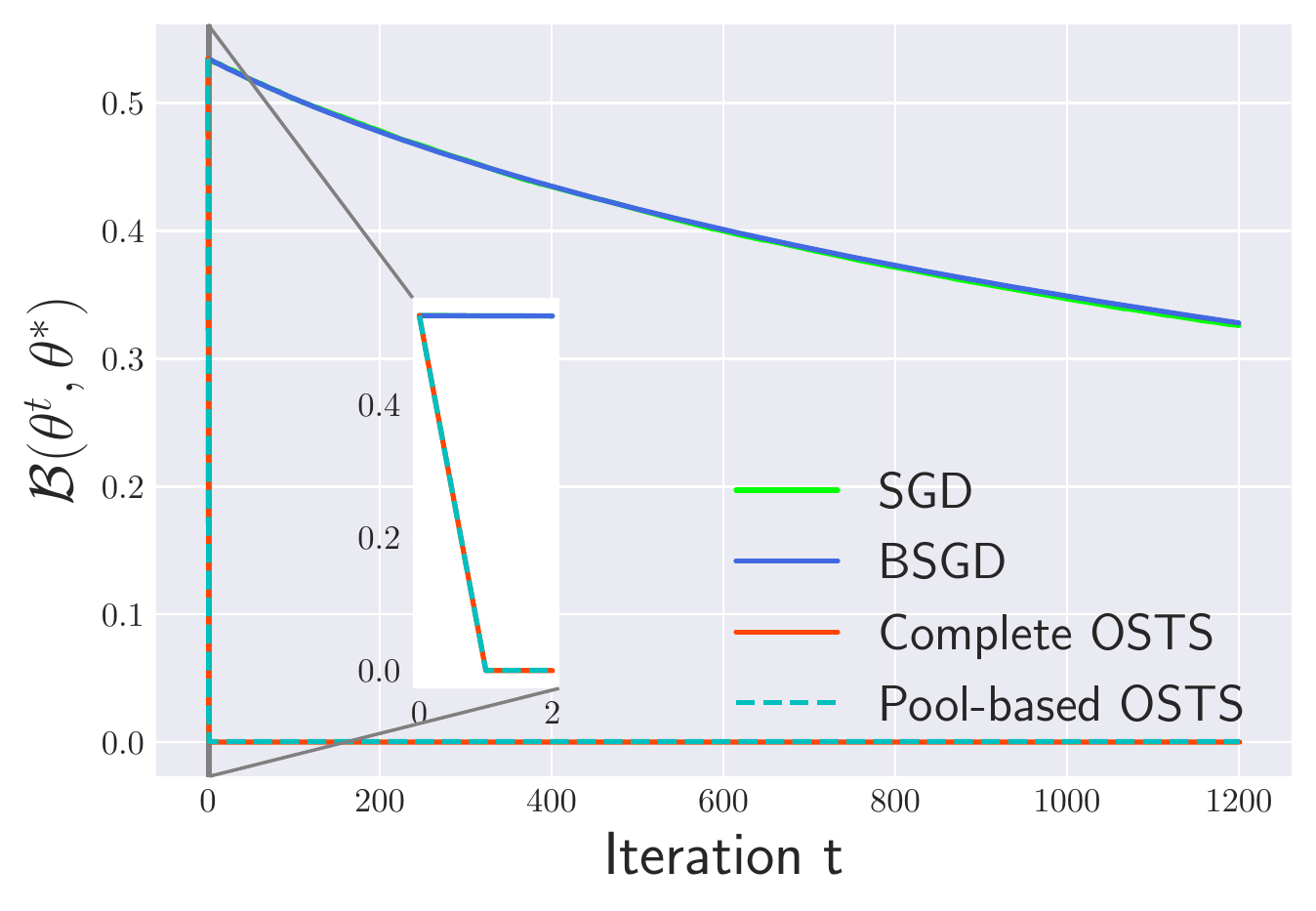}
			\includegraphics[width = 0.32\linewidth]{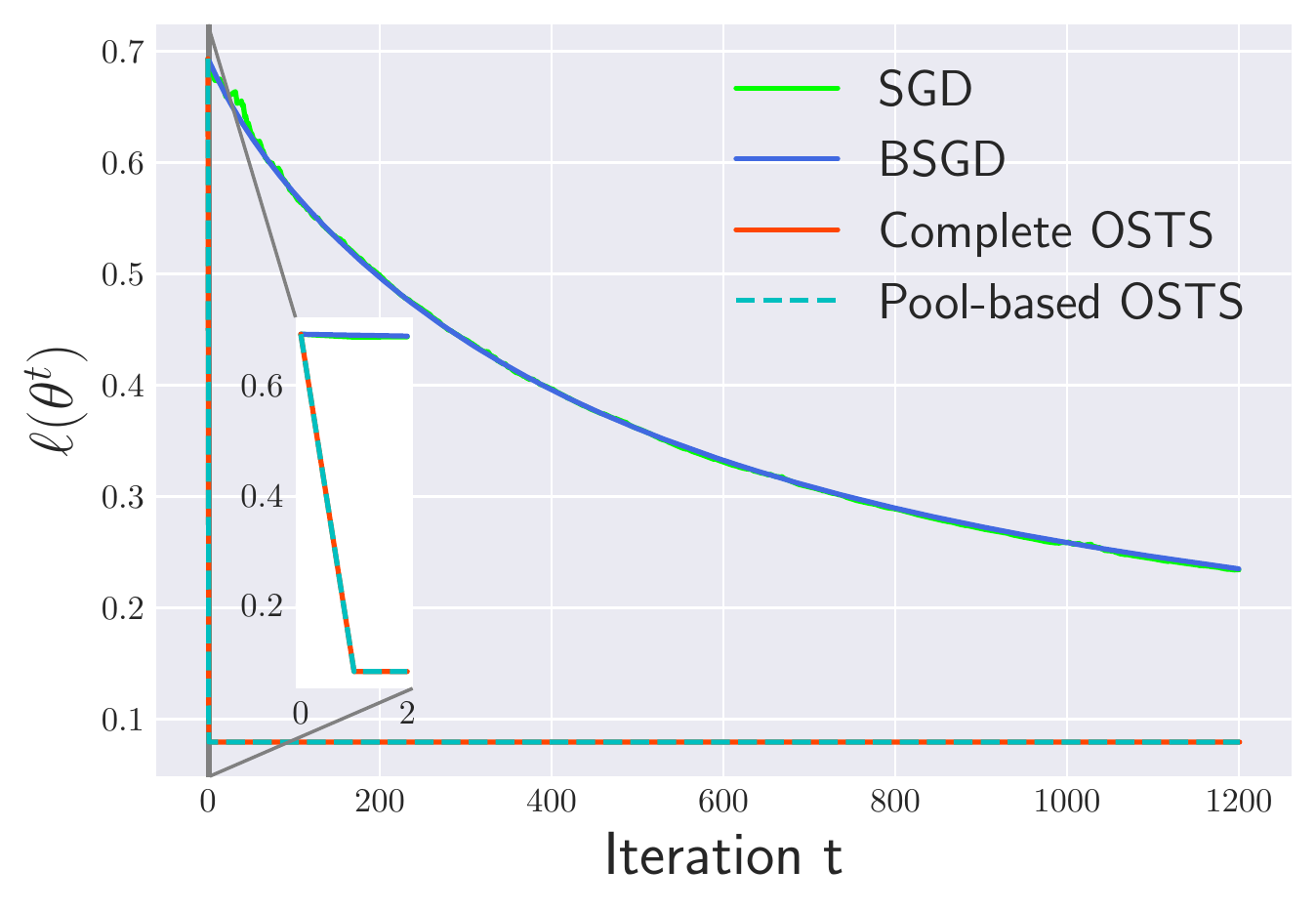}
			\includegraphics[width = 0.32\linewidth]{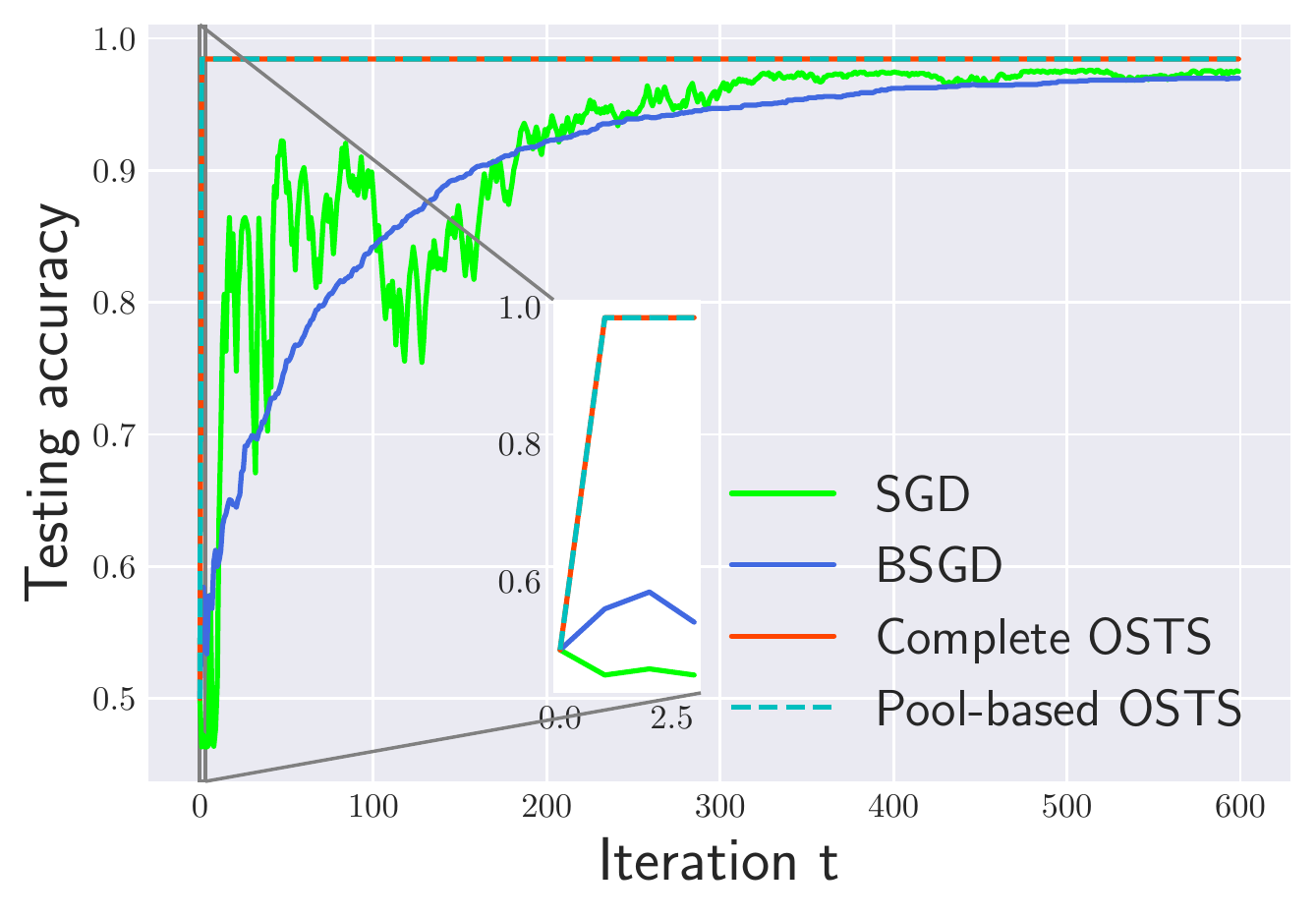}
		\end{minipage}
	}
	\hspace{-10pt}
	\subfigure[3 VS. 8]{
		\begin{minipage}[c]{\linewidth}
			\includegraphics[width = 0.32\linewidth]{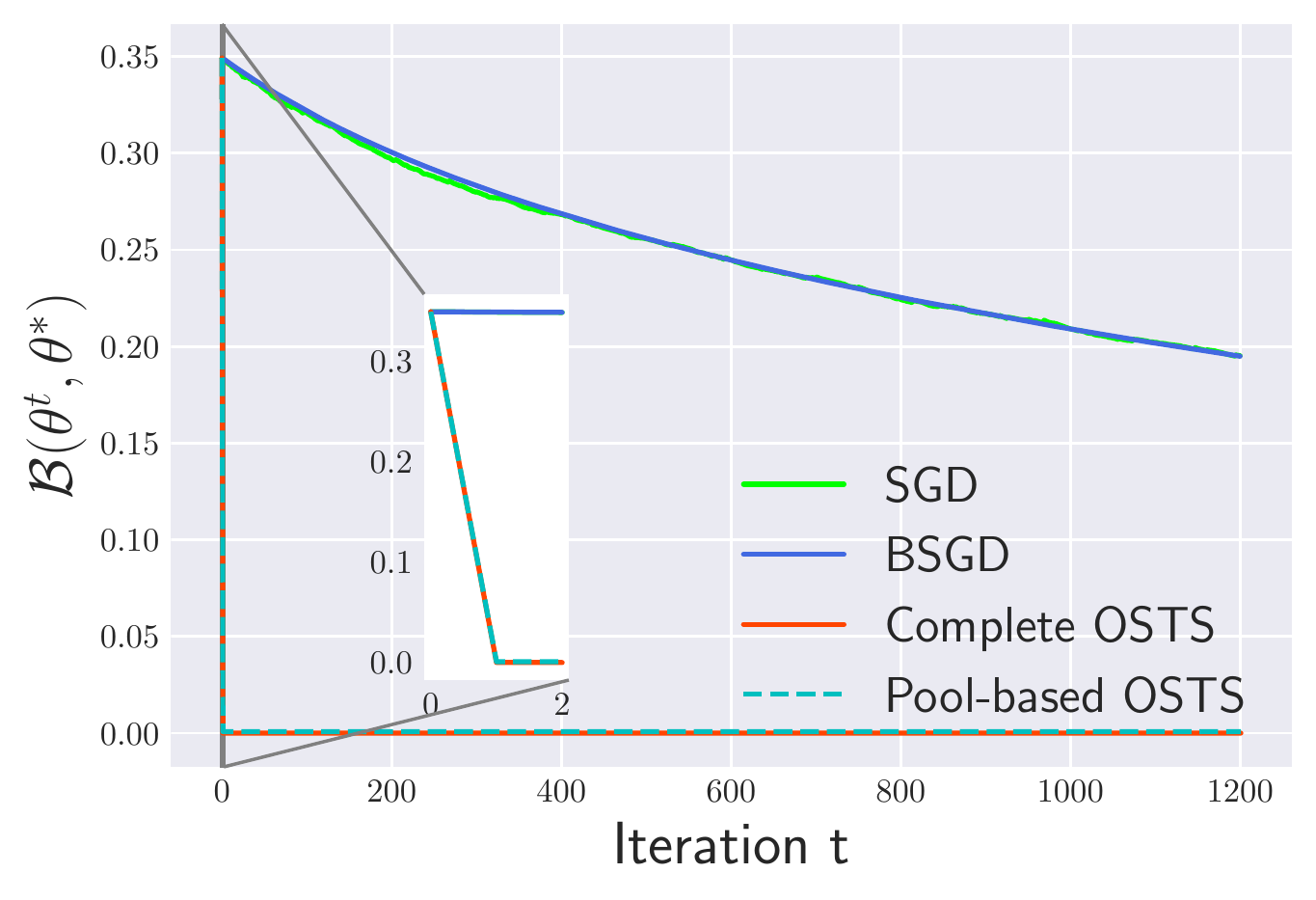}
			\includegraphics[width = 0.32\linewidth]{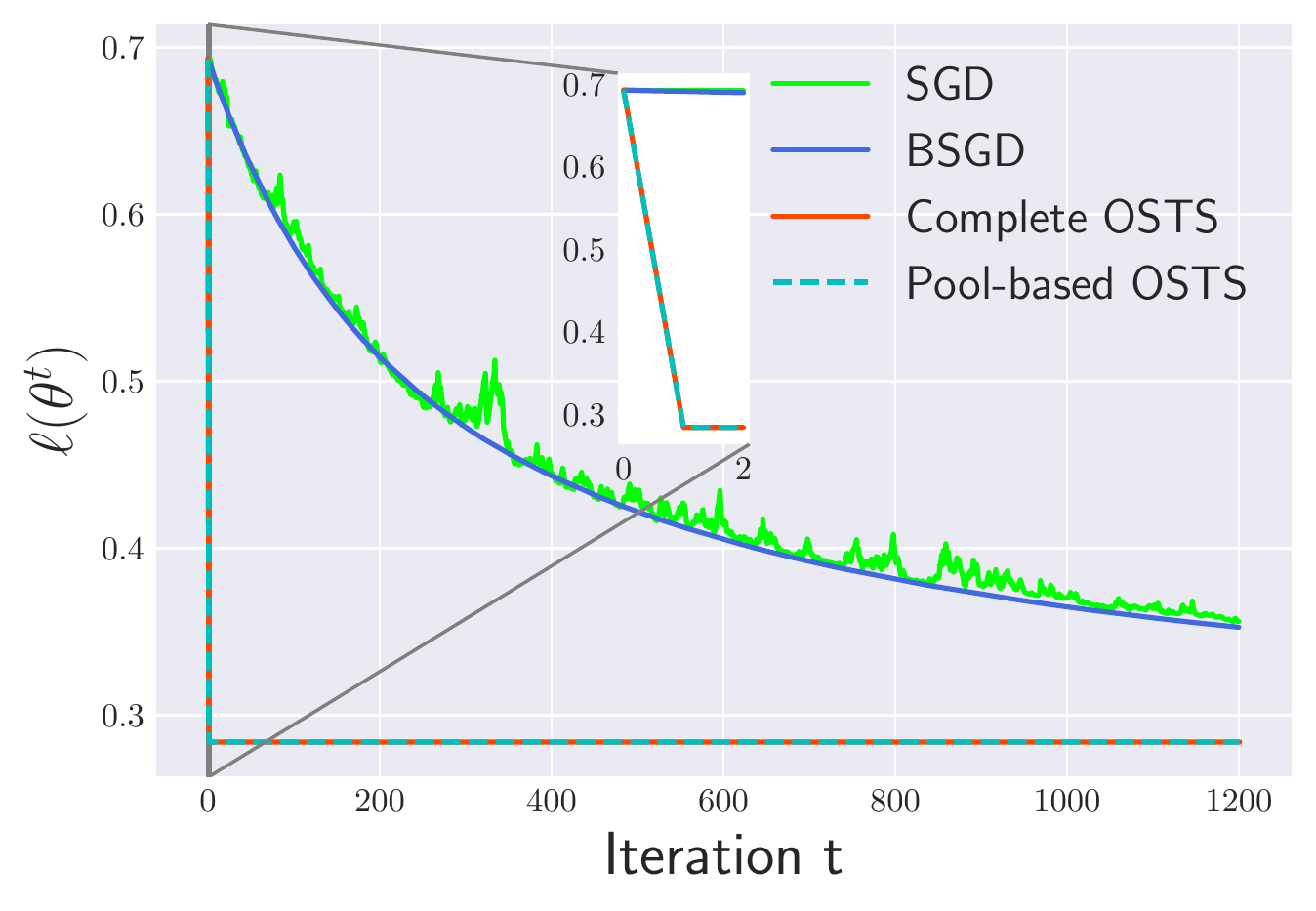}
			\includegraphics[width = 0.32\linewidth]{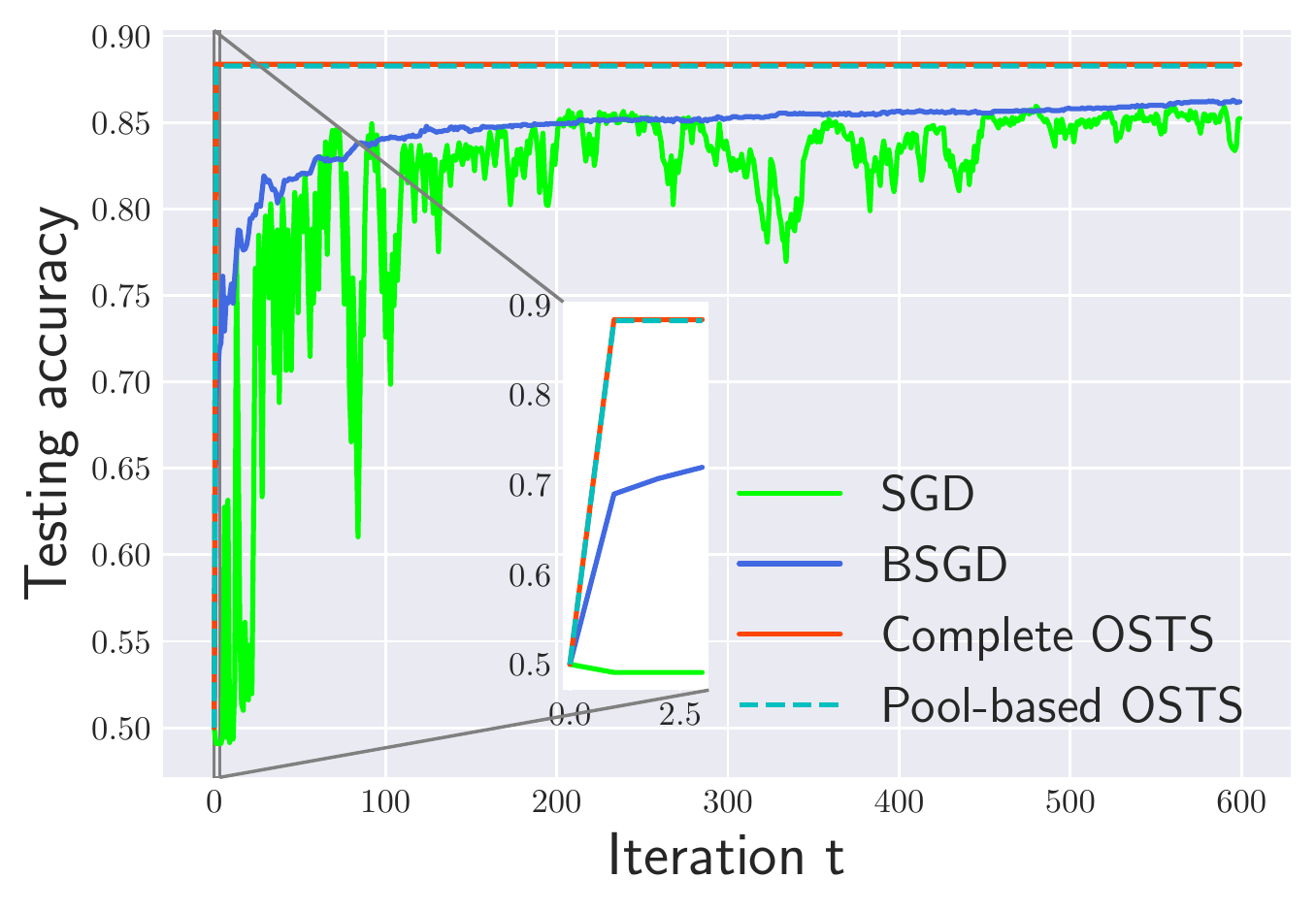}
		\end{minipage}
	}
	\caption{Teaching LR learners on MNIST dataset. Two $\mathcal{B}(\theta^t,\theta^*)$ figures illustrate that learners guided by OSTS under complete and pool-based teachers converge to $\theta^*$ at one step. Shown in testing loss and accuracy figures, OSTS under complete and pool-based teachers have lower testing loss and higher testing accuracy after one iteration, in contrast with gradual improvement of SGD and BSGD.}
	\label{b_l_ta_mnist}
\end{figure}

\subsubsection{Synthetic data}\label{atfd}

\begin{figure}[t]
	\begin{center}
		\advance\leftskip+1mm
		\renewcommand{\captionlabelfont}{\footnotesize}
		\vspace{-0.25in}  
		\includegraphics[width = .5\linewidth]{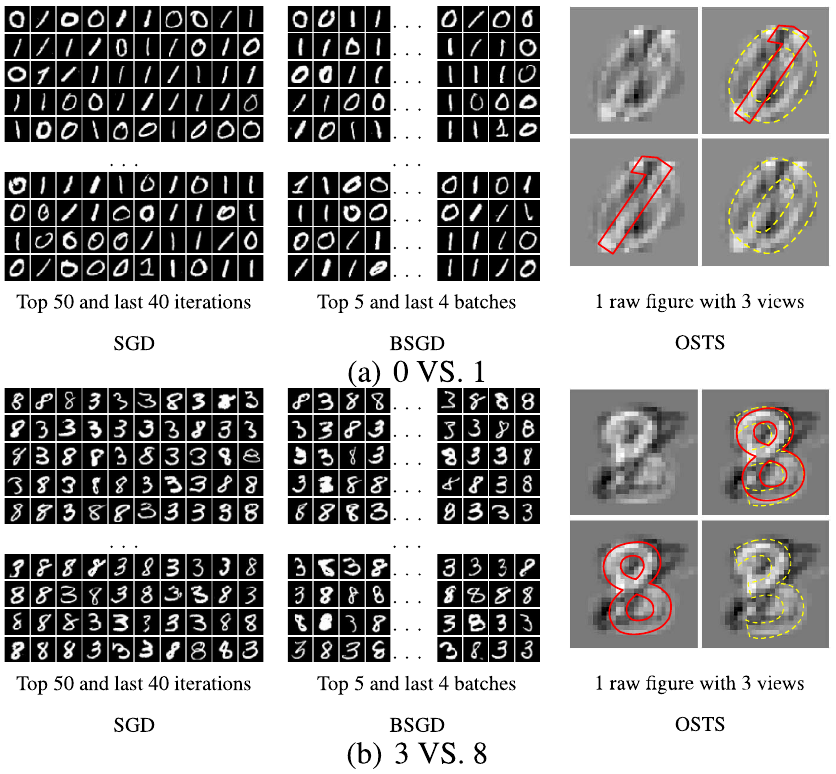}
		\vspace{-0.07in} 
		\caption{\footnotesize Teaching sets under different strategies on MNIST dataset in left-to-right and top-to-bottom ordering. There are no significant trends of the selection ordering under SGD and BSGD. Two optimal teaching examples $(x^*,y^*)$s are ambiguous.
		}
		\label{FOIPSinMNIST}
		\vspace{-0.11in} 
	\end{center}
\end{figure}

To show the general applicability of OSTS, we take LSR, SVM and LR learners mentioned in Section \ref{gl} as instances. LSR learners are taught regression tasks when SVM and LR learners are taught binary classification tasks. Besides, it is assumed that there are intercept terms in $\theta$ of these three learners. Out of robustness, we utilize stochastic loss $\ell_{sto}(\theta)$ on synthetic datasets defined as below. We set \texttt{Batch} in stochastic loss and BSGD as 100.

\begin{defn}[Stochastic Loss]
	Given a pool of examples, stochastic loss of a specific parameter $\theta$ is the average loss of a batch of stochastically chosen examples, which is evaluated at model $f(\theta,x)$.
\end{defn}

For instance, stochastic loss of $\theta^*$ is
$$\ell_{sto}(\theta^*)=\frac{1}{\text{batch}}\sum_{i=1}^{\text{batch}}\ell(f(\theta^*,x_i),y_i).$$ 

\textbf{LSR learner.} We set $\theta\in\Theta=\mathbb{R}^2$ so that we can visualize the training procedure of LSR learners. Specifically, $\theta^*=(-0.8,0.6)^T\text{ and }\theta^0=(0.1,0.3)^T$. Besides, teaching set of SGD and BSGD, $\mathcal{D}=\{(x_i,y_i)\}_{i=1}^{1000}$ contains 1000 examples where $x_i\sim \mathcal{U}(-10,10)$ and $y_i=\left<\theta^*,x_i\right>+\epsilon,\epsilon\sim \mathcal{N}(0,2^2)$. We use $\mathcal{U}$ and $\mathcal{N}$ to denote uniform and normal distribution, respectively. It follows from Theorem \ref{ost} that value of $\eta$ does not affect one-iteration convergence, we set $\eta$ for OSTS ($\eta_{\text{OSTS}}$) as 0.01 out of easy-visualized purpose, and $\eta$s for SGD and BSGD ($\eta_{\text{SGD\&BSGD}}$) are set 0.0001.

\textbf{SVM learner.} We set $\theta\in\Theta=\mathbb{R}^3$ for SVM learners with $\theta^*=(0.8,1,-0.6)^T \text{ and }
\theta^0=(-0.1,0.1,-0.7)^T$. Meanwhile, 500 positive and 500 negative teaching examples of SGD and BSGD are constructed as $x_+\sim \mathcal{N}\left((1.8,2.85)^T
,\text{I}_2\right)$ and $x_-\sim \mathcal{N}\left((-1.8,-1.65)^T,\text{I}_2\right)$ with 2-d identity matrix $\text{I}_2$. Similarly to LSR learners, $\eta_{\text{OSTS}}=0.1$ and static $\eta_{\text{SGD\&BSGD}}=0.0001$.

\textbf{LR learner.} $\Theta=\mathbb{R}^3$ as well. $\theta^*=(0.8,0.6,0.5)^T\text{ and }
\theta^0=(0.4,-0.2,0.46)^T$. Besides, 500 $x_+\sim \mathcal{N}\left((5,10)^T,\text{I}_2\right)$ and 500 $x_-\sim \mathcal{N}\left((-5,-10)^T,\text{I}_2\right)$. Same as previous two learners, $\eta_{\text{OSTS}}=0.1$ and static $\eta_{\text{SGD\&BSGD}}=0.0001$.

The performance of these all three learners is evaluated with learner bias $\mathcal{B}(\theta^t,\theta^*)$ and stochastic loss $\ell_{sto}(\theta^t)$ in combination with visualization of $\theta^t$ evolution. Learner bias and stochastic loss figures in Figure \ref{b_l_lp} clearly show that OSTS outperform than SGD and BSGD, demonstrating the effectiveness of our OSTS. In detail, $\mathcal{B}(\theta^t,\theta^*)$ and $\ell(\theta^t)$ of SGD and BSGD decrease gradually as iteration number grows thus, they are curves. More interestingly, since LSR, SVM and LR learners all are able to converge to $\theta^*$ within one iterative training on $(x^*,y^*)$ derived by OSTS, OSTS looks like seemingly vertical lines. Actually, the reason of this appearance is that the range of $x$ axis is just [0,1], which we have zoomed in. 2D and 3D figures of $\theta^t$ evolution are more vivid, which show that SGD and BSGD help learners approach $\theta^*$ progressively while learners of OSTS arrives $\theta^*$ at one step. Magical performance of OSTS shows the significant effectiveness of our theoretical proposition. Just as the fact in real world that there are many knowledgeable and sophisticated teachers qualified to help learners learn the knowledge with few efforts, OSTS qualifies every teacher to make it.

\begin{figure}[t]
	\subfigbottomskip=-6pt
	\subfigcapskip=0pt
	\centering
	\includegraphics[width = 0.32\linewidth]{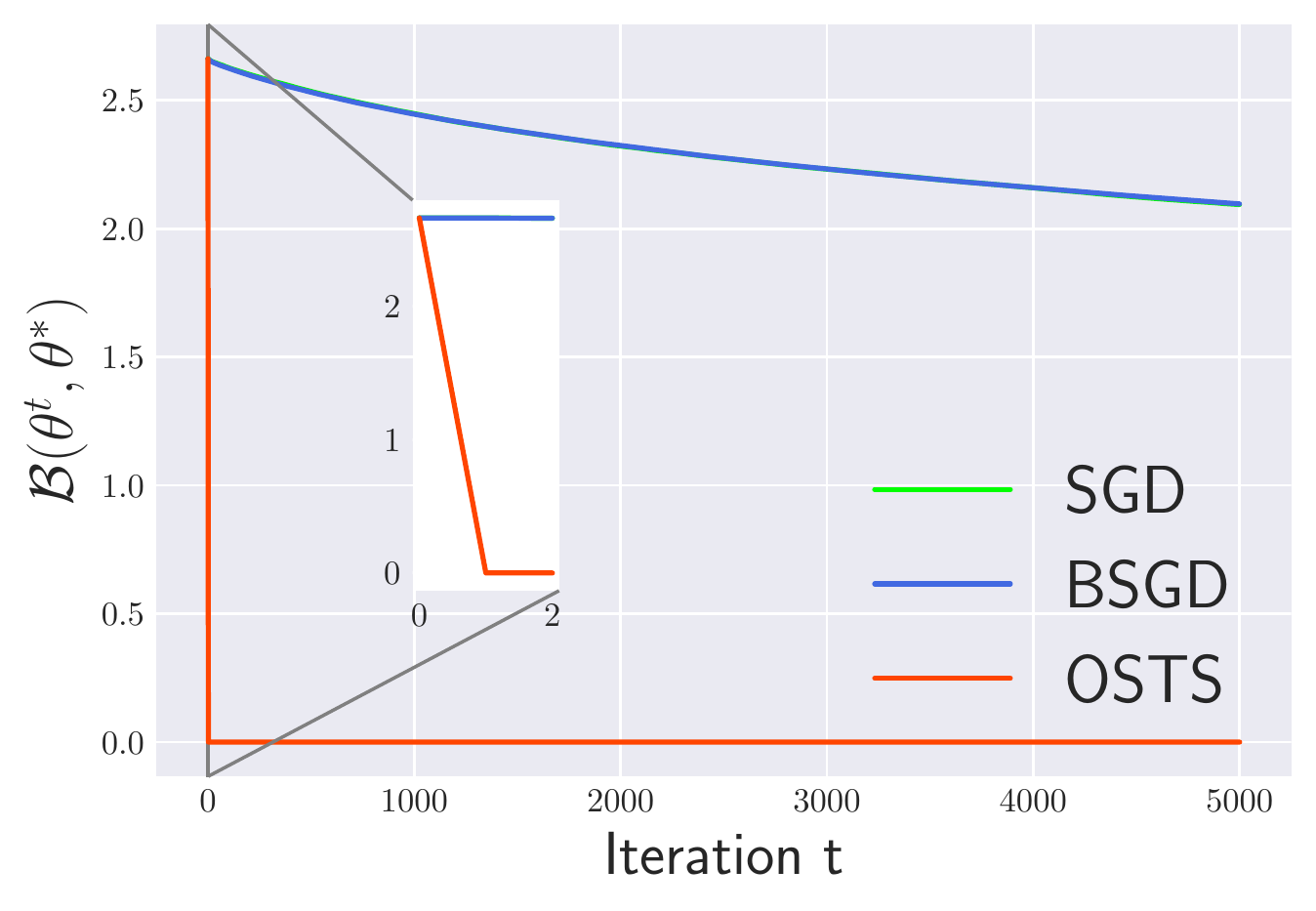}
	\includegraphics[width = 0.32\linewidth]{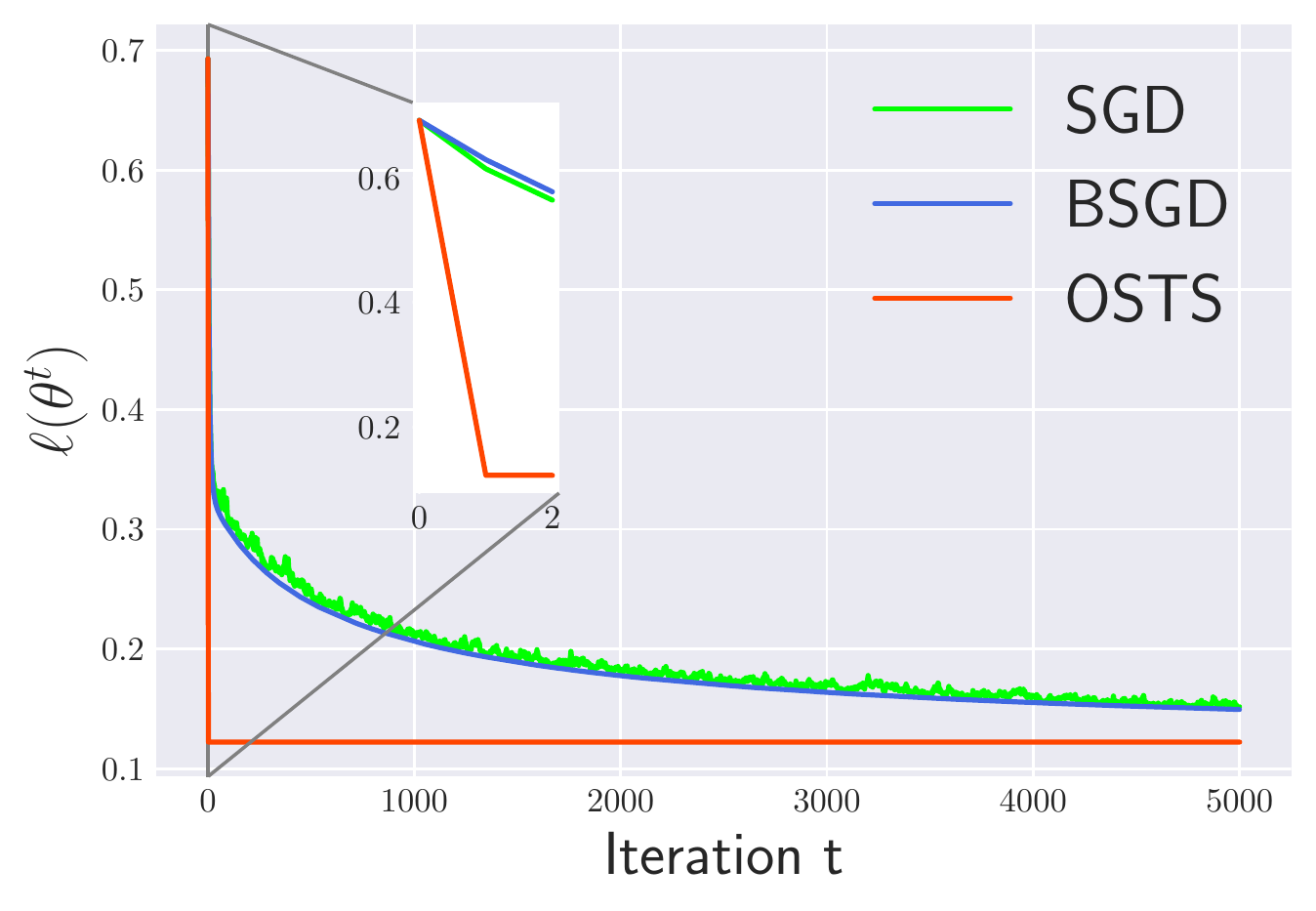}
	\includegraphics[width = 0.32\linewidth]{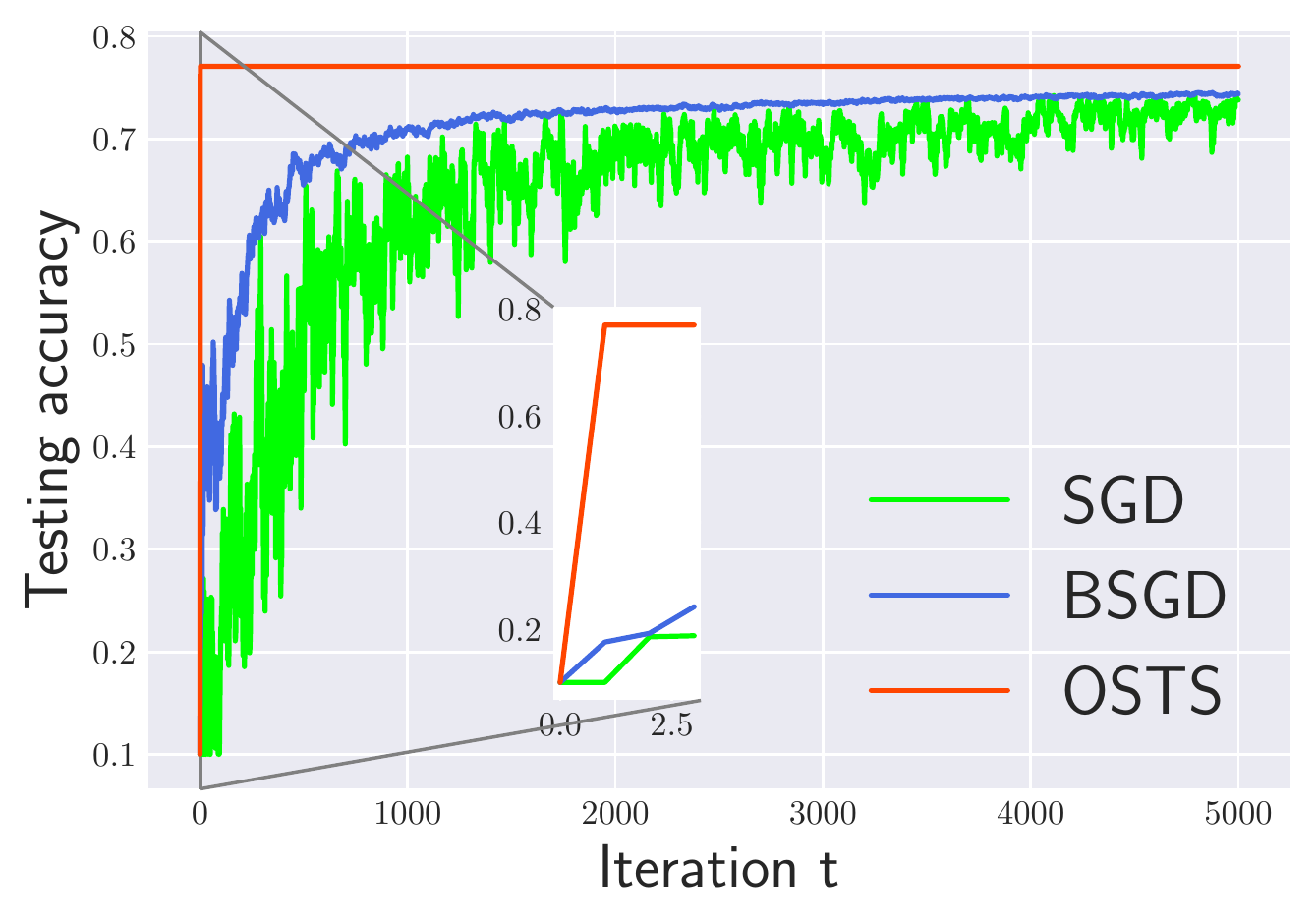}
	\vskip -0.15in
	\caption{Teaching LR learners on CIFAR10 dataset. $\mathcal{B}(\theta^t,\theta^*)$ curves show that SGD and BSGD approach $\theta^*$ slowly while OSTS reach $\theta^*$ at one step. Testing loss and accuracy figures illustrate that OSTS costs less (one iteration) than SGD and BSGD (more than 1000 iterations) to have competitive performance.}
	\label{b_l_ta_cifar10}
\end{figure}

\subsubsection{Real-world datasets}

\textbf{MNIST.} We consider binary classification tasks and LR learners, and test the performance of OSTS under both complete and pool-based teachers. We are mainly concerned with combinable teachers among pool-based teachers since it is more practical and flexible. First of all, random projection is hired to reduce 28*28 = 784D into 24D features (projection with a random matrix $\mathbb{R}^{784\times24}$). Based on 24D features, the LR learner with no primary knowledge is trained to do binary classification tasks, which are distinguishing 0 from 1 and 3 from 8. We randomly choose 48 teaching examples from the whole dataset of MNIST as the pool of combinable teacher $\mathcal{P}_{48}$, and the generative coefficient vector is 48D $\beta\in\mathbb{R}^{48}$. Besides, it is assumed that there are no intercept terms in $\theta$. $\eta_{\text{OSTS}}=\eta_{\text{SGD\&BSGD}}=0.0001$.

\begin{table}[b]
	\caption{Final results of different strategies on MNIST. "cOSTS" denotes OSTS under the complete teacher and "pOSTS" represents OSTS under the pool-based (combinable) teacher.}
	\label{mnisttb}
	\begin{center}
		\begin{small}
			\resizebox{\textwidth}{!}{
				\begin{tabular}{lcccccccc}
					\toprule
					&\multicolumn{4}{c}{0 VS. 1}   & \multicolumn{4}{c}{3 VS. 8}\\
					\cmidrule(lr){2-5}\cmidrule(lr){6-9}
					& SGD    & BSGD    & cOSTS    &pOSTS  & SGD    & BSGD   & cOSTS & pOSTS \\
					\midrule                        
					Learner Bias           & 0.3258 & 0.3279  & \textbf{0}       & 0.0003          & 0.1952 & 0.1950 & \textbf{0} & 0.0008\\
					Stochastic loss        & 0.2340 & 0.2350  & \textbf{0.0794}  & \textbf{0.0794} & 0.3560 & 0.3525 & \textbf{0.2837} & 0.2838\\
					Testing accuracy (\%)  & 97.68  & 97.49   & \textbf{98.44}   & \textbf{98.44}         & 86.44  & 87.00  & \textbf{88.36} & 88.26\\
					\bottomrule
				\end{tabular}
			}
		\end{small}
	\end{center}
	\vskip -0.1in
\end{table}

Figure \ref{b_l_ta_mnist} illustrates that SGD and BSGD help learners attain $\theta^*$ gradually while OSTS for both the complete and pool-based teacher did it at one step, which shows OSTS is more efficient. Interestingly, if teachers provide teaching sets at random, even they know $\theta^*$, learners usually fail to learn $\theta^*$. Instead, learners tend to converge to another suboptimal parameter ${\theta^*}'$, which results in higher testing loss and lower testing accuracy. These findings agree with the necessity of supervision from teachers. Final results are listed in Table \ref{mnisttb}. Besides, Figure \ref{FOIPSinMNIST} shows teaching sets of SGD and BSGD in combination with $(x^*,y^*)$s derived by OSTS. One can see examples in SGD and BSGD are randomly ordered. Aware of the linearity of random projection, we present $(x^*,y^*)$ of complete teachers which are ambiguous. The auxiliary lines of them offer two sights of $(x^*,y^*)$s in 0/1 or 3/8. Taking $(x^*,y^*)$ in 0/1 as an instance, it is like 1 and also close to 0, which is hard to classify. Actually, the label of $(x^*,y^*)$s are 1 and 8, respectively. Ambiguous appearance meets expectation since only after learning ambiguous examples, the learner is able to make a right judgment on indistinguishable samples, then we can say this learner is qualified to handle this classification problem. It is also consistent with active learning \cite{settles2009active} where learners query oracle with increasingly ambiguous examples.

\begin{table}[t]
	\footnotesize
	\caption{Structure of CNN.}
	\label{cnn}
	\vskip 0.15in
	\begin{center}
		\begin{small}
			\resizebox{\textwidth}{!}{
				\begin{tabular}{ccccccc}
					\toprule
					\multicolumn{2}{c}{Operation Layer} &  Filters Number & Filters Size & Stride & Padding & Output Size \\
					\midrule
					\multicolumn{2}{c}{Input image}        & -  &          -          & - & - & $3\times32\times32$\\
					\midrule
					\multirow{4}{*}{Block 1} & Convolution & 32 & $3\times3\times3$   & 1 & 1 & $32\times32\times32$\\
					& ReLU        & -  &          -          & - & - & $32\times32\times32$\\
					& Convolution & 64 & $32\times3\times3$  & 1 & 1 & $64\times32\times32$\\
					& ReLU        & -  &          -          & - & - & $64\times32\times32$\\
					& Max pooling & 1  &     $2\times2$      & 2 & 0 & $64\times16\times16$\\
					\midrule
					\multirow{4}{*}{Block 2} & Convolution & 128 & $64\times3\times3$   & 1 & 1 & $128\times16\times16$\\
					& ReLU        & -  &          -          & - & - & $128\times16\times16$\\
					& Convolution & 128 & $128\times3\times3$  & 1 & 1 & $128\times16\times16$\\
					& ReLU        & -  &          -          & - & - & $128\times16\times16$\\
					& Max pooling & 1  &     $2\times2$      & 2 & 0 & $128\times8\times8$\\
					\midrule
					\multirow{4}{*}{Block 3} & Convolution & 256 & $128\times3\times3$   & 1 & 1 & $256\times8\times8$\\
					& ReLU        & -  &          -          & - & - & $256\times8\times8$\\
					& Convolution & 256 & $256\times3\times3$  & 1 & 1 & $256\times8\times8$\\
					& ReLU        & -  &          -          & - & - & $256\times8\times8$\\
					& Max pooling & 1  &     $2\times2$      & 2 & 0 & $256\times4\times4$\\
					\midrule
					Final block              & Fully connected & -  &          -          & - & - & $10$\\
					\bottomrule
				\end{tabular}
			}
		\end{small}
	\end{center}
	\vskip -0.1in
\end{table}

\begin{table}[t]
	\caption{Final results of different strategies on CIFAR10.}
	\label{cifar10tb}
	\vskip 0.15in
	\begin{center}
		\begin{small}
			\begin{tabular}{lccc}
				\toprule
				& SGD    & BSGD    & OSTS             \\
				\midrule                        
				Learner Bias           & 2.0940 & 2.0965  & \textbf{0}       \\
				Testing loss           & 0.1517 & 0.1494  & \textbf{0.1222}  \\
				Testing accuracy (\%)  & 73.81  & 74.37   & \textbf{77.08}   \\
				\bottomrule
			\end{tabular}
		\end{small}
	\end{center}
	\vskip -0.1in
\end{table}

\textbf{CIFAR10.} We consider 10-class classification tasks and LR learners. We extend the shallow models to the deep ones, the convolutional neural network (CNN) whose structure is provided in Table \ref{cnn}. A CNN can be divided into feature extraction and classification parts. Specifically, feature extraction is comprised of blocks of convolutional, ReLU and pooling layers and classification refers to the final fully connected (FC) layers. We make use of feature extraction part to do feature engineering and derive input $x\in\mathcal{X}=\mathbb{R}^{4096}$. Based on 4096D features from feature extraction, the learner needs to learn the 10-class classifier identified by $\theta\in\Theta=\mathbb{R}^{4096\times10}$. Experiment results show that even with a large number of features, OSTS always works. We extract features at first 3 layer blocks of CNN, then train the fully connected layer on extracted features with the loss function (the sum of ten logistic loss functions for ten classes), finally teach this classifier to learners. $\eta_{\text{OSTS}}=\eta_{\text{SGD\&BSGD}}=0.001$.

Analogously, Figure \ref{b_l_ta_cifar10} illustrate that for 10-class classification tasks, OSTS teaches $\theta^*$ at one step while SGD and BSGD accomplish it step by step. Besides, it needs 10 distinct $(x^*,y^*)$ for OSTS to do that, and it needs more than 1000 teaching examples for SGD or BSGD. These show the high efficiency of OSTS. Besides, SGD has a poorer performance than BSGD since training on a single example (SGD) is not enough to distinguish 10 classes. The final results are shown in Table \ref{cifar10tb}.

\section{Conclusion}\label{cl}

In this paper, we propose a more intelligent machine teaching paradigm named one-shot machine teaching to achieve one-shot iteration, which is to search a singleton that can steer learners to converge after one iteration. Establishing a tractable mapping proved to be surjective from teaching sets to model parameters, the one-shot paradigm corroborates the existence of the optimal teaching set by which the optimal model can be learnt after one iteration. Relying on the surjective mapping, we develop a canonical design strategy of the optimal teaching set, namely one-shot teaching strategy whose teaching dimension and iterative teaching dimension are one, for both complete and pool-based (combinable, scalable and naive) teachers. We discuss the generalizability of our strategy w.r.t. the learner loss functions (e.g., square, hinge and logistic loss). The high efficiency of our strategy has been empirically verified, which also demonstrates the intelligence of this new teaching paradigm.

\section*{Acknowledgement}

This work was supported by the National Natural Science Foundation of China under Grant 62206108.

\bibliographystyle{elsarticle-num-names}
\bibliography{main.bib}

\end{document}